\documentclass{article}

\usepackage[nonatbib, preprint]{neurips_2020}

\usepackage[utf8]{inputenc} % allow utf-8 input
\usepackage[T1]{fontenc}    % use 8-bit T1 fonts
\usepackage{hyperref}       % hyperlinks
\usepackage{url}            % simple URL typesetting
\usepackage{booktabs}       % professional-quality tables
\usepackage{amsfonts}       % blackboard math symbols
\usepackage{microtype}      % microtypography

\usepackage{amsthm}
\usepackage{amsmath,amssymb}
\usepackage[pdftex]{graphicx}

\usepackage{algorithm, algpseudocode}
\newtheorem*{theorem*}{Theorem}
\newtheorem*{definition*}{Definition}
\newtheorem*{lemma*}{Lemma}
\newtheorem*{example*}{Example}
\newtheorem*{proposition*}{Propsotion}
\newtheorem*{corollary*}{Corollary*}

\newtheorem{theorem}{Theorem}

\newtheorem{lemma}{Lemma}
\newtheorem{corollary}{Corollary}

\graphicspath{{figs/}}
\makeatletter

\newcommand{\kakko}[1]{\left( #1 \right)}
\newcommand{\ckakko}[1]{\left\{ #1 \right\}}
\newcommand{\dkakko}[1]{\left[ #1 \right]}

\newcommand{\nkakko}[1]{\left\| #1 \right\|}
\newcommand{\zkakko}[1]{\left| #1 \right|}

\newcommand{\D}{\mathrm{d}}

\DeclareMathOperator*{\minimize}{\text{minimize} \quad}

\newcommand{\rank}[1]{\operatorname{rank}\left( #1 \right)}
\newcommand{\sign}[1]{\operatorname{sign} ( #1 )}

\newcommand{\abs}[1]{| #1 |}
\newcommand{\group}{{\mathcal{G}}}
\newcommand{\groupt}{{\mathcal{G}}}
\newcommand{\betah}{\beta}
\newcommand{\betag}{\beta^{\mathcal{G}}}
\newcommand{\betaz}{\beta^{\mathcal{G}}_{-0}}
\newcommand{\betat}{\beta}
\newcommand{\betatg}{\beta^{\mathcal{G}}}
\newcommand{\betatz}{\beta^{\mathcal{G}}_{-0}}
\newcommand{\Xg}{X^{\mathcal{G}}}
\newcommand{\xg}{x^{\mathcal{G}}}
\newcommand{\Xz}{X^{\mathcal{G}}_{-0}}
\newcommand{\Xtg}{X^{\mathcal{G}}}
\newcommand{\xtg}{x^{\mathcal{G}}}
\newcommand{\Xtz}{X^{\mathcal{G}}_{-0}}
\newcommand{\Xzt}{X^{\mathcal{\tilde{G}}}_{-0}}
\newcommand{\tila}{b}
\newcommand{\tilq}{q}
\newcommand{\gmin}{\underline{g}}
\newcommand{\gmax}{\overline{g}}

\title{Efficient Path Algorithms for Clustered Lasso and OSCAR}

\author{%
	Atsumori Takahashi \\
	The Institute of Statistical Mathematics \\
	\texttt{takahashi.atsumori@ism.ac.jp} \\
	 \And
	 Shunichi Nomura \\
	 The Institute of Statistical Mathematics \\
	 \texttt{nomura@ism.ac.jp} \\
}
\begin{document}
\maketitle

\begin{abstract}
In high dimensional regression, feature clustering by their effects on outcomes is often as important as feature selection.
For that purpose, clustered Lasso and octagonal shrinkage and clustering algorithm for regression (OSCAR)
are used to make feature groups automatically by pairwise $L_1$ norm and pairwise $L_\infty$ norm, respectively.
This paper proposes efficient path algorithms for clustered Lasso and OSCAR to construct solution paths with respect to their regularization parameters.
Despite too many terms in exhaustive pairwise regularization, their computational costs are reduced by using symmetry of those terms.
Simple equivalent conditions to check subgradient equations in each feature group are derived by some graph theories.
The proposed algorithms are shown to be more efficient than existing algorithms in numerical experiments.
\end{abstract}

\section{Introduction}
With the increasing prevalence of high-dimensional data in many fields in recent years, 
feature selection and clustering have become increasingly important. 
Lasso~\cite{tibshirani1996regression}, 
and its variants~\cite{tibshirani2005sparsity, hoefling2010path, tibshirani2011} 
have been developed as sparse regularization techniques for that purpose. %While the group Lasso~\cite{friedman2010group, simon2013group} predetermines candidates of feature groups,
This paper focuses on two kinds of feature-clustering regularization without any prior information on feature groups. %Consider a response vector $y \in \mathbb{R}^n$, a design matrix $X \in \mathbb{R}^{n\times p}$ and a coefficient vector $\beta \in \mathbb{R}^p$.
One is clustered Lasso~\cite{she2010} formulated by
\begin{equation} \label{clusteredlasso}
\minimize_{\beta}\ \ \frac{1}{2} \nkakko{y-X\beta}^2 +\lambda_1 \sum_{i=1}^{p} \abs{\beta_i} + \lambda_2 \sum_{j<k} \zkakko{\beta_j - \beta_k}, 
\end{equation}
where $y \in \mathbb{R}^n$ is a response vector, $X \in \mathbb{R}^{n\times p}$ is a design matrix, $\beta \in \mathbb{R}^p$ is a coefficient vector 
and $\lambda_{1}, \lambda_{2}$ are regularization parameters. 
The last term enforces coefficients to be similar or equal. %highly correlated coefficients to be equal to each other.
The other is octagonal shrinkage and clustering algorithm for regression (OSCAR)~\cite{bondell2008simultaneous} defined by
\begin{equation} \label{OSCAR}
\minimize_{\beta}\ \ \frac{1}{2} \nkakko{y-X\beta}^2 + \lambda_1 \sum_{i=1}^{p} \abs{\beta_i} + \lambda_2 \sum_{j<k} \max \ckakko{\abs{\beta_j}, \abs{\beta_k}},
\end{equation}
where the pairwise $L_{\infty}$ norm encourages absolute values of highly correlated coefficients to be zero or equal. 
Note that the pairwise $L_{\infty}$ norm in \eqref{OSCAR} can be converted into $L_1$ norm as 
$\max \ckakko{\abs{\beta_j}, \abs{\beta_k}} = (\abs{\beta_j- \beta_k} + \abs{\beta_j + \beta_k})/2$ and thus
OSCAR and clustered Lasso can be regarded as special cases of generalized Lasso~\cite{tibshirani2011}.

Fast methods to obtain a solution of coefficients $\beta$ with a fixed point of $[\lambda_{1}, \lambda_{2}]$ have been developed 
for clustered Lasso~\cite{lin2019efficient} and OSCAR~\cite{zhong2012efficient,bogdan2015, luo2019efficient}. 
However, on tuning the regularization parameters, a solution path of $\beta$ in a continuous range of $[\lambda_{1}, \lambda_{2}]$ is more preferable than 
a grid search on discrete values of $[\lambda_{1}, \lambda_{2}]$.
Algorithms to obtain such solution paths are called path algorithms and proposed in more general settings~\cite{tibshirani2011, zhou2013algorithms, arnold2016efficient} than clustered Lasso and OSCAR.
However, due to $p(p-1)/2$ pairwise regularization terms in \eqref{clusteredlasso} and \eqref{OSCAR}, 
those algorithms require too much computational costs on the order of $\mathcal{O}(np^2 \max \{n, p^2\} + T p^2 \max \{n, p^2\})$ 
for clustered Lasso and OSCAR where $T$ is the number of iterations in the algorithms.
In some special cases that, say, $X=I_p$ and $\lambda_1=0$, the solution path becomes simple and can be obtained fast
because the coefficients are only getting merged in turn as $\lambda_2$ increases~\cite{hocking2011clusterpath}.
Those cases can be extended to the weighted clustered Lasso with distance-decreasing weights~\cite{chiquet2017fast}.
An efficient path algorithm which obtains an approximate solution path with an arbitrary accuracy bound around a starting point is 
proposed for OSCAR with a general design matrix~\cite{gu2017groups}.

In this article, we propose novel path algorithms for clustered Lasso and OSCAR with a general design matrix.
The proposed algorithms can construct entire exact solution paths much faster than the existing ones. %efficiently by the complexity of $\mathcal{O}(np^2 + Tnp)$.
We specify two types of events which make breakpoints in solution paths and one auxiliary type of events in our algorithm.
Especially, we derive efficient methods to specify the event times using symmetry of the regularization terms as described in the later sections.

\section{Path algorithm for clustered Lasso}
In this section, we propose a solution path algorithm for clustered Lasso \eqref{clusteredlasso},
which yields a solution path $\betah(\eta)$ along regularization parameters 
$[\lambda_{1}, \lambda_{2}] = \eta[\bar{\lambda}_1, \bar{\lambda}_2]$ 
controlled by a single parameter $\eta > 0$ with a fixed direction $[\bar{\lambda}_1, \bar{\lambda}_2]$.
Hereafter, we assume %the design matrix is full rank i.e. 
$n\ge p$ and $\rank{X} = p$ to ensure that the objective function \eqref{clusteredlasso} is strictly convex.
Otherwise, the solution path might not be unique and continuous, making it difficult to track.
In such a case, a ridge penalty term $\varepsilon \nkakko{\beta}^2$ with a tiny weight $\varepsilon$ can be added 
to \eqref{clusteredlasso},
which is equivalent to extending the response vector $y$ and design matrix $X$ into $y^* = [y^\top,$\textrm{\boldmath $0$}$_p^\top]^\top$
and $X^* = [X^\top,\sqrt{\varepsilon} I_p]^\top$, to make the problem strictly convex~\cite{tibshirani2011, hu2015dual, arnold2016efficient, gaines2018algorithms}.

In \eqref{clusteredlasso}, the regularization terms encourage the coefficients to be zero or equal. 
Thus, from the solution $\betah(\eta)$ for a fixed regularization parameter $\eta > 0$, 
we define the set of {\it fused groups} $\group(\eta) = \{G_{\gmin}, G_{\gmin +1}, \dots, G_{\gmax}\}$ and
the {\it grouped coefficients} $\betag(\eta) = [\betag_{\gmin}, \betag_{\gmin +1},\dots,\betag_{\gmax}]^\top \in \mathbb{R}^{\gmax-\gmin+1}$ 
to satisfy the following statements:
\begin{itemize}
                \item $\bigcup_{g=\gmin}^{\gmax} G_g = \ckakko{1, \cdots, p}$, where $G_0$ may be an empty set but others may not.
                \item $\betag_{\gmin}<\cdots < \betag_{-1} < \betag_0 = 0 < \betag_1 < \cdots < \betag_{\gmax}$ and $\betah_i = \betag_g $ for $i\in G_g$.%,\; g=\gmin, \gmin +1,\dots,\gmax$.
\end{itemize}
Note that $\gmin \le 0 \le \gmax $ because the group $G_0$ exists as an empty set even if no zeros exist in the entries of $\beta(\eta)$.
Correspondingly, %Subsequently, %Accordingly,
we introduce the {\it grouped design matrix} $\Xg = [\xg_{\gmin}, \xg_{\gmin +1},\dots,\xg_{\gmax}] \in \mathbb{R}^{n\times (\gmax-\gmin+1)}$, 
where $\xg_g = \sum_{j \in G_g} x_j$ and $x_j$ is the $j$-th column vector of $X$. 
%We further define $\tilde{\beta}^\group = [\betag_{\gmin},\dots,\betag_{-1},\betag_1,\dots,\betag_{\gmax}]^\top \in \mathbb{R}^{\gmax-\gmin}$
%a vector of elements in $\betag$ except $\betag_0 = 0$ and $\tilde{X}^\group = [x^\group_{\gmin},\dots,x^\group_{-1},x^\group_1,\dots,x^\group_{\gmax}] \in \mathbb{R}^{n\times (\gmax-\gmin)}$ 
%the corresponding design matrix.

\subsection{Piecewise linearity} % 2.1
Because the objective function of \eqref{clusteredlasso} is strictly convex, its solution path would be a continuous piecewise linear function~\cite{hoefling2010path,tibshirani2011}.
Specifically, as long as the signs and order of the solution $\beta(\eta)$ are conserved as in the set $\group$ of fused groups defined above,
the problem \eqref{clusteredlasso} can be reduced to the following quadratic programming:
\begin{equation*} %\label{clusteredlasso-deformation}
\minimize_{\beta}\ \ \frac{1}{2} \nkakko{y - \Xg \betag}^2 + \eta \bar{\lambda}_1 \sum_{g=\gmin}^{\gmax} p_g \sign{\betag_g} \betag_g + \eta \bar{\lambda}_2 \sum_{g=\gmin}^{\gmax} p_g r_g  \betag_g,
\end{equation*}
where $p_g$ is the cardinality of $G_g$, $q_g = p_{\gmin} + \cdots + p_{g-1}$
is the number of coefficients smaller than $\betag_g$, and $r_g = q_g - (p - q_{g+1})$.
Hence, because $\betag_0$ is fixed at zero, 
the nonzero elements of the grouped coefficients $\betaz(\eta)= [\betag_{\gmin},\dots,\betag_{-1},\betag_1,\dots,\betag_{\gmax}]^\top$ are obtained by
\begin{equation} \label{clusterelasso-theta}
\betaz(\eta) = \dkakko{\kakko{\Xz}^\top \Xz}^{-1} \dkakko{\eta a +\kakko{\Xz}^\top  y}, 
\end{equation}
where $\Xz = [\xg_{\gmin},\dots,\xg_{-1},\xg_1,\dots,\xg_{\gmax}] $ %is the submatrix of $\Xg$ corresponding to $\betaz$
and $a = [a_{\gmin},\dots,a_{-1},a_1,\dots,a_{\gmax}]^\top \in \mathbb{R}^{\gmax-\gmin}$, %is a vector such that
$a_g = -\bar{\lambda}_1 p_g \sign{\betag_g} - \bar{\lambda}_2 p_g r_g$.
Thus, the solution path $\beta(\eta)$ moves linearly along $\betag(\eta)$ as defined above until the set $\group$ of fused groups changes.

\subsection{Optimality condition} % 2.2
%Fused groups can be changed by only two kinds of events; one is that adjacent groups are fused into one group and the other is that one group is split into smaller groups.
%Whereas the former one is detected by tracking the solution path until some of the group coefficients collide,
%the latter one occurs when the linear solution path $\beta(\eta)$ in \eqref{clusterelasso-theta} violates the optimality condition of \eqref{clusterelasso}.
%In this section, we first show the optimality condition and then derive an equivalent condition easier to check.

In this subsection, we show the optimality conditions of \eqref{clusteredlasso} and then derive a theorem to check them efficiently.
Owing to $L_1$ norms, the optimality condition involves their subgradients as follows: 
\begin{equation} 
x_i^\top \kakko{\Xg \betag - y} + \tau_{i0} + \lambda_2 r_g +  
\sum_{j \in G_g \setminus\ckakko{i}} \tau_{ij} = 0, \qquad i \in G_g,%,\, g\neq 0$}, \label{clusteredlasso-subgradientequation} \\
%x_i^\top \kakko{\Xg \betag - y} + \tau_{i0} + \lambda_2 r_0 +  
%\sum_{j \in G_0 \setminus\ckakko{i}} \tau_{ij} &= 0 \quad \mbox{for $i \in G_0$}, 
\label{clusteredlasso-subgradientequation}
\end{equation}
where $\tau_{i0} \in [-\lambda_1,\lambda_1]$ is a subgradient of $\lambda_1 \zkakko{\beta_i}$
which takes $\lambda_1 \sign{\betag_g}$ if $\betag_g \neq 0$, % such that $\tau_{i0} \in [-\lambda_1,\lambda_1]$ if $\beta_i=0$ and $\tau_{i0} = \lambda_1 \sign{\beta_i}$ otherwise. %is a subgradient of $\lambda_1 \zkakko{\beta_i}$ if $\beta_i=0$ and $\tau_{i0} = \lambda_1 \sign{\beta_i}$ otherwise %$\tau_{ij} \in [-\lambda_2,\lambda_2]$ is a subgradient of $\lambda_2 \zkakko{\beta_i - \beta_j}$ subject to the constraints $\tau_{ij} + \tau_{ji} = 0$.
and $\tau_{ij} \in [-\lambda_2,\lambda_2] \; (i\neq j,\; i,j\in G_g) $ is subject to the constraint $\tau_{ij} = -\tau_{ji}$,
which implies that $\tau_{ij}$ is a subgradient of $\lambda_2 \zkakko{\beta_i - \beta_j}$ with respect to $\beta_i$ when $\beta_i = \beta_j$.
For an overview of subgradients, see e.g.~\cite{bertsekas1999nonlinear}.

%When subgradients satisfying \eqref{clusteredlasso-subgradientequation} no longer exist in a group $G_g$, 
%the group has to be split into smaller groups so that each group satisfies 
%the optimal conditions \eqref{clusteredlasso-subgradientequation}~\cite{hoefling2010path}.
%In contrast, when the conditions \eqref{clusteredlasso-subgradientequation0} are violated, 
%the group $G_0$ would change in two ways; one is splitting into smaller groups and the other is leaving from zero altogether.

Though it is not straightforward to check the optimality conditions including subgradients, we can derive their equivalent conditions, which are easier to verify. % by using the symmetry of the conditions \eqref{clusteredlasso-subgradientequation} and \eqref{clusteredlasso-subgradientequation0} under exchanging indices within each group. %Especially, a group can be split only into the higher and lower groups of the first term $x_i^\top \kakko{X^\group \betag - y}$ in \eqref{clusteredlasso-subgradientequation} and \eqref{clusteredlasso-subgradientequation0}.
For the group $G_0$ of zeros, if we assume $m=p_0\ge 1$ and denote by $f_1 \ge f_2 \ge \cdots \ge f_m$ the sorted values of $x_i^\top (\Xg \betag - y) + \lambda_2 r_g $ over $i \in G_0$,
we can propose the following theorem to check condition \eqref{clusteredlasso-subgradientequation}.
\begin{theorem} \label{theorem1}
	There exist $\tau_{i0} \in \dkakko{-\lambda_1, \lambda_1}$ and $\tau_{ij} = -\tau_{ji} \in \dkakko{-\lambda_2, \lambda_2}$ such that
	\begin{equation} \label{lemma1-subgradequation}
	f_i + \tau_{i0} + \sum_{j \in \{1,\dots,m\} \setminus\ckakko{i}}\tau_{ij} =0, \qquad i=1,\dots,m,
	\end{equation}
	if and only if	
	\begin{align}
	\sum_{j=1}^{k} f_j &\le \lambda_1 k + \lambda_2k(m-k), \quad & k=1,\dots, m, \label{lowerbound} \\
	\sum_{j=k+1}^{m} f_j &\ge -\lambda_1 (m-k) - \lambda_2 k(m-k), \quad & k=0,\dots, m-1. \label{upperbound}
	\end{align}
\end{theorem}
In Appendix \ref{appendix-proof1}, we prove this theorem by using symmetry of the regularization terms and
the idea in~\cite{hoefling2010path} that existence of subgradients in a fused Lasso problem
can be checked through a maximum flow problem. 

%In the proof of theorem 1 provided in the supplementary material, we generalize the idea in~\cite{hoefling2010path} that existence of subgradients in fused Lasso problems
%can be checked through the maximum flow problem and use the symmetry of . 

For a nonzero group $G_g$, let $m=p_g$ and $f_1 \ge f_2 \ge \cdots \ge f_m$ denote the sorted values of 
$x_i^\top (\Xg \betag - y) + \lambda_{1} \sign{\betag_g} + \lambda_2 r_g $ over $i \in G_g$. %$\sum_{i=1}^n f_i = 0$ from \eqref{clusteredlasso-groupwise} and 
Then, the following corollary of Theorem \ref{theorem1}, which is derived by fixing $\lambda_1$ at zero, can be used to check optimality condition \eqref{clusteredlasso-subgradientequation}.
\begin{corollary} \label{corollary1}
	There exist $\tau_{ij} = -\tau_{ji} \in \dkakko{-\lambda_2, \lambda_2}$ such that
	\begin{equation} \label{coro1}
	f_i + \sum_{j \in \{1,\dots,m\} \setminus\ckakko{i}}\tau_{ij} =0, \qquad i=1,\dots,m,
	\end{equation}
	if and only if %$\sum_{j=1}^{k} f_j \le \lambda_2 k(m-k)$ for $k=1,\dots,m-1$ and $\sum_{j=1}^{m} f_j = 0$.
	\begin{equation} \label{coro2}
	\sum_{j=1}^{k} f_j \le \lambda_2 k(m-k), \quad  k=1,\dots,m-1,
	\end{equation}
        and $\sum_{j=1}^{m} f_j = 0$.
\end{corollary}

\subsection{Events in path algorithm} % 2.3
In this subsection, we specify the events and their occurrence times in our path algorithm, an outline of which is presented in Section \ref{pathalgorithm}.
First, we define two types of events that change the set $\group$ of fused groups.
One is the {\it fusing event}, in which adjacent groups are fused when their coefficients collide.
The other is the {\it splitting event}, in which 
a group is split into smaller groups satisfying the optimality condition for each group when the condition is violated within the group.

%The outline of the whole algorithm is shown in Section 4.
%Fused groups can be changed by only two kinds of events. 
%Whereas the fusing event is detected by tracking the solution path until some of the group coefficients collide,
%the splitting event occurs when the linear solution path $\beta(\eta)$ in \eqref{clusterelasso-theta} violates the optimality condition shown in the previous section.
%Especially, a group can be split only into the higher and lower groups of the first term $x_i^\top \kakko{X^\group \betag - y}$ in \eqref{clusteredlasso-subgradientequation} and \eqref{clusteredlasso-subgradientequation0}.

Because the nonzero coefficients $\betaz(\eta)$ moves according to \eqref{clusterelasso-theta} until the groups change and $\betag_0(\eta)\equiv 0$,
two adjacent groups $G_g$ and $G_{g+1}$ have to be fused at time $\Delta^{\text{fuse}}_g(\eta)$ from $\eta$ given by
\begin{equation*}
\Delta^{\text{fuse}}_g(\eta) = \begin{cases}  \frac{-\betag_{g+1}(\eta) + \betag_{g}(\eta)}{\frac{\D \betag_{g+1}}{\D \eta} - \frac{\D \betag_{g}}{\D \eta}} & \text{if $\frac{\D \betag_{g+1}}{\D \eta} < \frac{\D \betag_{g}}{\D \eta}$}, \\
\displaystyle \infty & \text{otherwise},
\end{cases}
\end{equation*}
where $\frac{\D \betag_{g}}{\D \eta}$ is the slope of $\betag_g(\eta)$ %in \eqref{clusterelasso-theta}
with nonzero elements $\frac{\D \betaz}{\D \eta} = [(\Xz)^{\top} \Xz]^{-1}a$
and $\frac{\D \betag_{0}}{\D \eta}= 0$.
%Thus fusing time is defined as the minimum time $\Delta^{\text{fuse}} = \min_{g} \Delta_{g}^{\text{fuse}}$ of them.

To specify the splitting events, let $o(j) \in \{1,\dots,p\}$ denote the order of coefficients
such that %if $\beta_{o(i)} < \beta_{o(j)}$ or $\beta_{o(i)} = \beta_{o(j)} \cap x_i^\top (\Xg \betag - y) > x_j^\top (\Xg \betag - y)$, we have $o(i) < o(j)$. %let $o(g,1),\dots,o(g,p_g) \in G_g$ denote all the indices in $G_g$ sorted in descending order of $x_i^\top (\Xg \betag - y)$ to apply Theorem \ref{theorem1} and Corollary 1.
$G_g = \{ o(q_g+1),\dots,o(q_{g+1}) \} $ and $x_{o(q_g+1)}^\top (\Xg \betag - y) \ge \cdots \ge x_{o(q_{g+1})}^\top (\Xg \betag - y)$. %Then, we obtain $G_g = \{ o(q_g+1),\dots,o(q_{g+1}) \} $ and $x_{o(q_g+1)}^\top (\Xg \betag - y) \ge \cdots \ge x_{o(q_{g+1})}^\top (\Xg \betag - y)$.
From Corollary \ref{corollary1}, the optimality condition \eqref{clusteredlasso-subgradientequation} in a nonzero group $G_g$ holds if and only if
\begin{equation*} %\label{clusteredlasso-nonzero}
x_{\underline{o}(g,k)}^\top \kakko{\Xg \betag - y} + \lambda_{1} k \sign{\betag_g} + \lambda_2 k r_g \le \lambda_2 k(p_g-k), \qquad  k=1,\dots, p_g,
\end{equation*}
where $x_{\underline{o}(g,k)} = \sum_{j=q_g+1}^{q_g+k} x_{o(j)}$.
Hence, this condition fails with $k$, and $G_g$ has to be split into $\{o(q_g+1),\dots,o(q_g+k)\}$ and $\{o(q_g+k+1),\dots,o(q_{g+1})\}$
at time $\Delta_{g, k}^{\text{split}}$ from $\eta$ given by
%When the condition \label{clusteredlasso-nonzero} get failed with $k$,
%the group $G_g$ has to be split into at least two groups $\{o(g,1),\dots,o(g,k)\}$ and $\{o(g,k+1),\dots,o(g,p_g)\}$.
%It occurs at time $\Delta_{g, k}^{\text{split}}$ from $\eta$ given by
\begin{equation*}
	\Delta_{g, k}^{\text{split}}= \begin{cases} \frac{ x_{\underline{o}(g,k)}^\top (y - \Xg \betag(\eta))+ \lambda_1 k \sign{\betag_g} + \lambda_2 k\dkakko{r_g-(p_g-k)}}
        {\sigma(g, k)} & \text{if $\sigma(g, k) > 0$}, \\
	\infty & \text{otherwise}, 
	\end{cases}
\end{equation*}
where $\sigma(g, k) = x_{\underline{o}(g,k)}^\top \Xg \frac{\D \betag}{\D \eta} + \bar{\lambda}_1 k \sign{\betag_g} + \bar{\lambda}_2 k\dkakko{r_g-(p_g-k)}$.

From Theorem \ref{theorem1}, the optimality condition \eqref{clusteredlasso-subgradientequation} in $G_0$ holds if and only if
\begin{align}
	x_{\underline{o}(0,k)}^\top \kakko{\Xg \betag - y} + \lambda_2 k r_0 &\le \lambda_1 k + \lambda_2k(p_0-k), \qquad  k=1,\dots, p_0, \label{clusteredlasso-zeroupper} \\
	x_{\overline{o}(0,k)}^\top \kakko{\Xg \betag - y} + \lambda_2 (p_0-k) r_0 &\ge -\lambda_1 (p_0-k) - \lambda_2 k(p_0-k), \quad  k=0,\dots, p_0-1, \label{clusteredlasso-zerolower}
%	&\sum_{j=m+1}^{p_0} x_{o(g,j)}^\top \kakko{X^\group \betag - y} + \lambda_2 (p_0-m) r_0 \ge -\lambda_1 (p_0-m) - \lambda_2 m(p_0-m) \quad  \mbox{ $m=0,\dots, p_0-1$}. \label{clusteredlasso-zerolower}
\end{align}
where $x_{\overline{o}(0,k)} = \sum_{j=q_0+k+1}^{q_1} x_{o(j)}$.
Therefore, for $k = 1,\dots,p_0$, the condition \eqref{clusteredlasso-zeroupper} fails with $k$ and the coefficients $\beta_{o(q_0+1)},\dots,\beta_{o(q_0+k)}$ 
deviate from zero in the negative direction
at time $\Delta_{0,-k}^{\text{split}}$ from $\eta$ given by
\begin{equation*}
	\Delta_{0,-k}^{\text{split}}= \begin{cases} \frac{x_{\underline{o}(0,k)}^\top (y - \Xg \betag(\eta)) - \lambda_1 k  + \lambda_2 k \dkakko{r_0-(p_0-k)}}{\sigma(0, -k)} & \text{if $\sigma(0, -k) > 0$}, \\
	\infty & \text{otherwise}, 
	\end{cases}
\end{equation*}
where $\sigma(0, -k) = x_{\underline{o}(0,k)}^\top \Xg \frac{\D \betag}{\D \eta} - \bar{\lambda}_1 k  + \bar{\lambda}_2 k \dkakko{r_0-(p_0-k)}$.
Similarly, for $k = 0,\dots,p_0-1$, the condition \eqref{clusteredlasso-zerolower} fails with $k$ and the coefficients $\beta_{o(q_0+k+1)},\dots,\beta_{o(q_1)}$
deviate from zero in the positive direction
at time $\Delta_{0,k}^{\text{split}}$ from $\eta$ given by
\begin{equation*}
	\Delta_{0,k}^{\text{split}}= \begin{cases} \frac{x_{\overline{o}(0,k)}^\top (y - \Xg \betag(\eta))+ \lambda_1 (p_0-k) +  \lambda_2 (p_0-k)(r_0+k)}{\sigma(0, k)} & \text{if $\sigma(0, k) < 0$}, \\
	\infty & \text{otherwise}, 
	\end{cases}
\end{equation*}
where $\sigma(0, k) = x_{\overline{o}(0,k)}^\top \Xg \frac{\D \betag}{\D \eta} + \bar{\lambda}_1 (p_0-k) +  \bar{\lambda}_2 (p_0-k)(r_0+k)$.
%For convenience, we further define $\Delta_{0, k}^{\text{split}} = \min\{ \Delta_{0-, k}^{\text{split}},\Delta_{0+,k-1}^{\text{split}} \}$ for $k=1,\dots,p_0$.

In our path algorithm, it is necessary to define another internal event in which 
the order $o(1),\dots,o(p)$ of indices changes within a group; we call this event the {\it switching event}.
For $k \in \{1,\dots,p-1\}\setminus \{q_{\gmin +1},\dots,q_{\gmax}\}$,
the indices assigned to $o(k)$ and $o(k+1)$ are switched by reversal of the inequality $x_{o(k)}^\top (\Xg \betag - y) \ge x_{o(k+1)}^\top (\Xg \betag - y)$  
at time $\Delta_{k}^\text{switch}$ from $\eta$ given by
\begin{equation*}
\Delta_{k}^\text{switch} = \begin{cases} \frac{\kakko{x_{o(k)}^\top-x_{o(k+1)}^\top} \dkakko{y-\Xg\betag(\eta)}}{\kakko{x_{o(k)}^\top-x_{o(k+1)}^\top} \Xg \frac{\D \betag}{\D \eta}} & \text{if $(x_{o(k)}^\top-x_{o(k+1)}^\top) \Xg \frac{\D \betag}{\D \eta} < 0$}, \\
\infty & \text{otherwise}.
\end{cases}
\end{equation*}

\section{Path algorithm for OSCAR}
In this section, we propose a solution path algorithm for OSCAR \eqref{OSCAR}, which is derived in a manner similar to that for clustered Lasso.
We use the same symbols and variables as in the preceding section, which have similar but slightly different definitions.
Our algorithm constructs a solution path of $\betat(\eta)$ along regularization parameters 
$[\lambda_{1}, \lambda_{2}] = \eta[\bar{\lambda}_1, \bar{\lambda}_2]$.
We also assume $n\ge p$ and $\rank{X} = p$ to ensure strict convexity. %or add an $L_2$ regularization term $\varepsilon \nkakko{\beta}^2$ to \eqref{OSCAR} to ensure strict convexity.

In contrast to clustered Lasso, 
the regularization terms in \eqref{OSCAR} encourage the absolute values of the coefficients to be zero or equal. 
Hence, from the solution $\betat(\eta)$, 
we define the fused groups $\groupt(\eta) = \ckakko{G_0, G_1, \dots, G_{\gmax}}$ and
the {\it grouped absolute coefficients} $\betatg(\eta) = [\betatg_0, \betatg_1,\dots,\betatg_{\gmax}]^\top \in \mathbb{R}^{\gmax+1}$ 
to satisfy the following statements:
\begin{itemize}
                \item $\bigcup_{g=0}^{\gmax} G_g = \ckakko{1, \cdots, p}$, where $G_0$ may be an empty set but others may not.
                \item $\betatg_0=0< \betatg_1 < \cdots < \betatg_{\gmax}$ and $\abs{\betat_i} = \betatg_g $ for $i\in G_g$.%,\; g=0, 1,\dots,\gmax$.
\end{itemize}
Correspondingly, 
we define the {\it signed grouped design matrix} 
$\Xtg = [\xtg_0, \xtg_1,\dots,\xtg_{\gmax}] \in \mathbb{R}^{n\times (\gmax+1)}$, 
where $\xtg_g = \sum_{j \in G_g} \sign{\beta_j} x_j$. 

\subsection{Piecewise linearity} % 3.1
Because the objective function of \eqref{OSCAR} is strictly convex, 
its solution path would be a continuous piecewise linear function as well as that of clustered Lasso.
As long as the grouping $\group$ of the solution $\beta(\eta)$ are conserved as defined above,
the problem \eqref{OSCAR} can be reduced to the following quadratic programming:
\begin{equation*} %\label{OSCAR-deformation}
\minimize_{\beta}\ \ \frac{1}{2} \nkakko{y - \Xtg \betatg}^2 + \eta \bar{\lambda}_1 \sum_{g=1}^{\gmax} p_g \betatg_g + \eta \bar{\lambda}_2 \sum_{g=1}^{\gmax} p_g \kakko{\tilq_g + \frac{p_g - 1}{2}}  \betatg_g,
\end{equation*}
where $p_g$ is the cardinality of $G_g$ and $\tilq_g = p_0 + \cdots + p_{g-1}$.
Hence, because $\betag_0$ is fixed at zero, 
the nonzero elements of the absolute grouped coefficients $\betatz(\eta)= [\betatg_1,\dots,\betatg_{\gmax}]^\top$ are obtained by
\begin{equation} \label{OSCAR-theta}
\betatz(\eta) = \dkakko{\kakko{\Xtz}^\top \Xtz}^{-1} \dkakko{\eta \tila +\kakko{\Xtz}^\top  y}, 
\end{equation}
where $\Xtz = [\xtg_1,\dots,\xtg_{\gmax}] $ %is the submatrix of $\Xtg$ corresponding to $\betatz$ and
and $\tila = [\tila_1,\dots,\tila_{\gmax}]^\top \in \mathbb{R}^{\gmax}$, %is a vector such that
$\tila_g = -\bar{\lambda}_1 p_g - \bar{\lambda}_2 p_g (\tilq_g + \frac{p_g - 1}{2})$.

\subsection{Optimality condition} 
In this subsection, we show the optimality conditions of \eqref{OSCAR} and then derive a theorem to check them efficiently.
For a nonzero group $G_g \; (g\neq 0)$, the optimality condition can be described as follows:
\begin{equation} \label{OSCAR-subgradientequation}
s_i x_i^\top \kakko{\Xtg \betatg - y} + \lambda_{1} + \lambda_2 \tilq_g +  \sum_{j \in G_g \setminus\ckakko{i}} \kakko{\tau_{ij}+\frac{\lambda_2}{2}} = 0, \qquad i \in G_g,
\end{equation}
where $s_i = \sign{\beta_i}$ and $\tau_{ij}\in [-\lambda_2/2,\lambda_2/2]\; (i\neq j,\; i,j\in G_g) $ are subject to the constraints $\tau_{ij} + \tau_{ji} = 0$,
which imply that $\tau_{ij}+\frac{\lambda_2}{2}$ is a subgradient of 
$\lambda_2 \max \ckakko{\abs{\beta_i}, \abs{\beta_j}} = \frac{\lambda_2}{2}(\abs{s_i\beta_i - s_j\beta_j} + s_i\beta_i + s_j\beta_j)$ 
with respect to $s_i \beta_i$ when $s_i \beta_i = s_j \beta_j>0$.
Then, if we denote by $f_1 \ge f_2 \ge \cdots \ge f_m$ the sorted values of $s_i x_i^\top (\Xtg \betatg - y) + \lambda_1 + \lambda_2 (\tilq_g + \frac{p_g - 1}{2}) $ over $i \in G_g$, 
we can apply Corollary \ref{corollary1} to the condition \eqref{OSCAR-subgradientequation} to specify when it fails as in the next subsection.

For the group $G_0$ of zeros, the optimality condition is given by
\begin{equation} \label{OSCAR-subgradientequation0}
x_i^\top \kakko{\Xtg \betatg - y} + \xi_{i0} + \sum_{j \in G_0 \setminus\ckakko{i}} \xi_{ij} = 0, \qquad i \in G_0,
\end{equation}
where $\xi_{i0} \in [-\lambda_1,\lambda_1]$ is a subgradient of $\lambda_1 \zkakko{\beta_i}$ when $\beta_i=0$ 
and $\xi_{ij} \in [-\lambda_2,\lambda_2] \; (i\neq j,\; i,j\in G_0) $ are subject to the constraints $\abs{\xi_{ij}} + \abs{\xi_{ji}} \le \lambda_2$,
which implies that $\xi_{ij}$ is a subgradient of the $L_{\infty}$ penalty $\lambda_2 \max \ckakko{\abs{\beta_i}, \abs{\beta_j}}$ with respect to $\beta_i$ when $\beta_i = \beta_j = 0$.
When we assume $m=p_0\ge 1$ and denote by $f_1,\dots,f_m$ the values of $x_i^\top (\Xtg \betatg - y )$ in $i \in G_0$ sorted as $\abs{f_1}\le \abs{f_2}\le \cdots \le \abs{f_m}$, %by their absolute values in descending order,
we can propose the following theorem to check condition \eqref{OSCAR-subgradientequation0}.
\begin{theorem} \label{theorem2}
	There exist $\xi_{i0} \in \dkakko{-\lambda_1, \lambda_1}$ and $\xi_{ij} \in \dkakko{-\lambda_2, \lambda_2}$ such that $\abs{\xi_{ij}} + \abs{\xi_{ji}} \le \lambda_2$ and 
	\begin{equation} \label{lemma2-subgradequation}
	f_i + \xi_{i0} + \sum_{j \in \{1,\dots,m\} \setminus\ckakko{i}}\xi_{ij} =0, \qquad i=1,\dots,m, %\\
%        , \quad 1\le i<j\le m,
	\end{equation}
	if and only if
	\begin{equation} \label{lemma2-condition}
	\sum_{j=k+1}^{m} \abs{f_j} \le \lambda_1 (m-k) + \lambda_2 \frac{(m-k)(m+k-1)}{2}, \qquad  k=0,\dots, m-1.
	\end{equation}
\end{theorem}
The proof of this theorem is provided in Appendix \ref{appendix-proof2}.

\subsection{Events in path algorithm} % 3.3
In our path algorithm for OSCAR, fusing, splitting and switching events are defined similarly to those for clustered Lasso.
From \eqref{OSCAR-theta} and $\betatg_0(\eta)\equiv 0$,
the time $\Delta^{\text{fuse}}_g(\eta)$ from $\eta$ to the fusing event in which two adjacent groups $G_g$ and $G_{g+1}$ have to be fused is given by
\begin{equation*}
\Delta^{\text{fuse}}_g(\eta) = \begin{cases}  \frac{-\betag_{g+1}(\eta) + \betag_{g}(\eta)}{\frac{\D \betag_{g+1}}{\D \eta} - \frac{\D \betag_{g}}{\D \eta}} & \text{if $\frac{\D \betag_{g+1}}{\D \eta} < \frac{\D \betag_{g}}{\D \eta}$}, \\
\displaystyle \infty & \text{otherwise},
\end{cases}
\end{equation*}
where $\frac{\D \betatg_{g}}{\D \eta}$ is the slope of $\betatg_g(\eta)$ %in \eqref{clusterelasso-theta}
with nonzero elements $\frac{\D \betatz}{\D \eta} = [(\Xz)^{\top} \Xz]^{-1}\tila$
and $\frac{\D \betatg_{0}}{\D \eta}= 0$.

%and their occurrence times in our path algorithm whose outline is shown in Section 4.
%First, we define two kinds of events which change the set $\group$ of fused groups.
%One is the {\it fusing event} that adjacent groups are fused into one group when their coefficients collide.
%The other is the {\it splitting event} that, when the optimality condition is violated within a group, 
%the group is split into smaller groups satisfying the condition for each group.

%The outline of the whole algorithm is shown in Section 4.
%Fused groups can be changed by only two kinds of events. 
%Whereas the fusing event is detected by tracking the solution path until some of the group coefficients collide,
%the splitting event occurs when the linear solution path $\beta(\eta)$ in \eqref{clusterelasso-theta} violates the optimality condition shown in the previous section.
%Especially, a group can be split only into the higher and lower groups of the first term $x_i^\top \kakko{X^\group \betag - y}$ in \eqref{clusteredlasso-subgradientequation} and \eqref{clusteredlasso-subgradientequation0}.

To specify the splitting events, let $o(j) \in \{1,\dots,p\}$ denote the order of coefficients
such that %if $\beta_{o(i)} < \beta_{o(j)}$ or $\beta_{o(i)} = \beta_{o(j)} \cap x_i^\top (\Xg \betag - y) > x_j^\top (\Xg \betag - y)$, we have $o(i) < o(j)$. %let $o(g,1),\dots,o(g,p_g) \in G_g$ denote all the indices in $G_g$ sorted in descending order of $x_i^\top (\Xg \betag - y)$ to apply Theorem \ref{theorem1} and Corollary 1.
$G_g = \{ o(q_g+1),\dots,o(q_{g+1}) \} $ and $s_{o(q_g+1)} x_{o(q_g+1)}^\top (\Xg \betag - y) \ge \cdots \ge s_{o(q_{g+1})} x_{o(q_{g+1})}^\top (\Xg \betag - y)$ %Then, we obtain $G_g = \{ o(q_g+1),\dots,o(q_{g+1}) \} $ and $x_{o(q_g+1)}^\top (\Xg \betag - y) \ge \cdots \ge x_{o(q_{g+1})}^\top (\Xg \betag - y)$.
for each group, where $s_i$ is defined by
\begin{equation*} %\label{OSCAR-nonzero}
s_i = \begin{cases}  \sign{\beta_i} & \text{if $\beta_i \neq 0$}, \\ - \sign{x_i^\top (\Xtg \betatg - y)} & \text{if $\beta_i = 0$}. \end{cases}
\end{equation*} %is the sign of $\beta_i$ used to apply Corollary \ref{corollary1} and Theorem \ref{theorem2}. % For a nonzero group $G_g$, let $o(g,1),\dots,o(g,p_g) \in G_g$ denote all the indices in $G_g$ sorted in descending order of $s_i x_i^\top (\Xtg \betatg - y)$ to specify the splitting event.
Note that, when $\beta_i$ hits or leaves zero, the value of $s_i$ does not change while its definition changes.
Then, from Corollary \ref{corollary1}, optimality condition \eqref{OSCAR-subgradientequation} in a nonzero group $G_g$ holds if and only if
\begin{equation*} %\label{OSCAR-nonzero}
x_{\underline{o}(g,k)}^\top \kakko{\Xg \betag - y} + \lambda_{1} k + \lambda_2 k \kakko{q_g + \frac{p_g - 1}{2}} \le \frac{\lambda_2}{2} k(p_g-k), \qquad  k=1,\dots, p_g,
\end{equation*}
where $x_{\underline{o}(g,k)} = \sum_{j=q_g+1}^{q_g+k} s_{o(j)} x_{o(j)}$.
Hence, %the condition \label{OSCAR-nonzero} fails with $k$ and 
$G_g$ is split into $\{o(q_g+1),\dots,o(q_g+k)\}$ and $\{o(q_g+k+1),\dots,o(q_{g+1})\}$
when this condition fails with $k$ at time $\Delta_{g, k}^{\text{split}}$ from $\eta$ given by
\begin{equation*}
	\Delta_{g, k}^{\text{split}}= \begin{cases} \frac{ x_{\underline{o}(g,k)}^\top (y - \Xg \betag(\eta)) + \lambda_1 k + \lambda_2 k (q_g + \frac{k-1}{2})}{\sigma(g, k)} & \text{if $\sigma(g, k) > 0$}, \\
	\infty & \text{otherwise}, 
	\end{cases}
\end{equation*}
and $\sigma(g, k) = x_{\underline{o}(g,k)}^\top \Xg \frac{\D \betag}{\D \eta} + \bar{\lambda}_1 k + \bar{\lambda}_2 k (q_g + \frac{k-1}{2})$.

%When $G_0$ is not empty, let $o(0,1),\dots,o(0,p_0) \in G_0$ denote %all the indices in $G_0$ sorted in descending order of $s_i x_i^\top (\Xtg \betatg - y)$ where $s_i = - \sign{x_i^\top (\Xtg \betatg - y)}$ for $i\in G_0$.
For $G_0$, from Theorem \ref{theorem2}, the optimality condition \eqref{OSCAR-subgradientequation0} holds if and only if
\begin{equation*} %\label{OSCAR-zero}
x_{\overline{o}(0,k)}^\top \kakko{\Xg \betag - y} \ge - \lambda_1(p_0-k) - \lambda_2\frac{(p_0-k)(p_0 + k - 1)}{2}, \qquad  k=1,\dots, p_0,
\end{equation*}
where $x_{\overline{o}(0,k)} = \sum_{j=k+1}^{p_0} s_{o(0,j)}x_{o(0,j)}$.
Hence, the coefficients $\beta_{o(k+1)},\dots,\beta_{o(p_0)}$ deviate from zero %are split from $G_0$, leave zero and compose a new group $G_1$
when this condition fails with $k$ at time $\Delta_{0,k}^{\text{split}}$ from $\eta$ given by
\begin{equation*}
	\Delta_{0,k}^{\text{split}}= \begin{cases} \frac{x_{\overline{o}(0,k)}^\top (y - \Xg \betag(\eta))+ \lambda_1 (p_0-k)  + \lambda_2 \frac{(p_0-k)(p_0 + k - 1)}{2}}{\sigma(0, k)} & \text{if $\sigma(0, k) < 0$}, \\
	\infty & \text{otherwise}, 
	\end{cases}
\end{equation*}
and $\sigma(0, k) = x_{\overline{o}(0,k)}^\top \Xg \frac{\D \betag}{\D \eta} + \bar{\lambda}_1 (p_0-k)  + \bar{\lambda}_2 \frac{(p_0-k)(p_0 + k - 1)}{2}$.
%Note that, for indices $i$ transferred from $G_0$ into $G_1$, $s_i$ does not change by the transfer though its definition changes into $s_i = \sign{\beta_i}$, and vice versa.

The switching event of the order $o(1),\dots,o(p)$ is needed for OSCAR as well.
For $k \in \{1,\dots,p-1\}\setminus \{q_1,\dots,q_{\gmax}\}$,
the indices assigned to $o(k)$ and $o(k+1)$ are switched by reversal of the inequality $s_{o(k)} x_{o(k)}^\top (\Xtg \betatg - y) \ge s_{o(k+1)} x_{o(k+1)}^\top (\Xtg \betatg - y)$ %for $k = 1,\dots,p_g-1$ 
at time $\Delta_{k}^\text{switch}$ from $\eta$ given by
\begin{equation*} 
\Delta_{k}^\text{switch} = \begin{cases} \frac{(s_{o(k)} x_{o(k)}^\top-s_{o(k+1)} x_{o(k+1)}^\top) [y-\Xtg\betatg(\eta)]}{(s_{o(k)} x_{o(k)}^\top-s_{o(k+1)} x_{o(k+1)}^\top) \Xtg \frac{\D \betatg}{\D \eta}} & \text{if $(s_{o(k)} x_{o(k)}^\top-s_{o(k+1)} x_{o(k+1)}^\top) \Xg \frac{\D \betag}{\D \eta} < 0$}, \\
%\Delta_{k}^\text{switch} = \begin{cases} \frac{\kakko{s_{o(k)} x_{o(k)}-s_{o(k+1)} x_{o(k+1)}}^\top \dkakko{y-\Xtg\betatg(\eta)}}{\kakko{s_{o(k)} x_{o(k)}-s_{o(k+1)} x_{o(k+1)}}^\top \Xtg \frac{\D \betatg}{\D \eta}} & \text{\!\!\! if $\kakko{s_{o(k)} x_{o(k)}-s_{o(k+1)} x_{o(k+1)}}^\top \Xg \frac{\D \betag}{\D \eta} > 0$}, \\
\infty & \text{otherwise}.
\end{cases}
\end{equation*}
Additionally, since $0\ge s_{o(1)} x_{o(1)}^\top (\Xtg \betatg - y) \ge \cdots \ge s_{o(p_0)} x_{o(p_0)}^\top (\Xtg \betatg - y)$, 
we need to add another case to the switching event in OSCAR, in which the sign $s_{o(1)} = - \sign{x_{o(1)}^\top (\Xtg \betatg - y)}$ reverses
at time $\Delta_{0}^\text{switch}$ from $\eta$ given by
\begin{equation*} 
\Delta_{0}^\text{switch} = \begin{cases} \frac{x_{o(1)}^\top \dkakko{y-\Xtg\betatg(\eta)}}{x_{o(1)}^\top \Xtg \frac{\D \betatg}{\D \eta}} & \text{if $s_{o(1)} x_{o(1)}^\top \Xtg \frac{\D \betatg}{\D \eta} > 0$}, \\
\infty & \text{otherwise}.
\end{cases}
\end{equation*}

\section{Path algorithm and complexity} \label{pathalgorithm}
The outline of our path algorithms for clustered Lasso and OSCAR are shown in Algorithm~\ref{oscar-algorithm}.
Though the variables are defined differently for clustered Lasso and OSCAR,
both algorithms have the same types of events and thus can be described in a common format.

As for the computational cost, $\mathcal{O}(np^2)$ time is required to obtain the initial solution $\beta^{(0)} = (X^\top X)^{-1}X^\top y$.
By using a block matrix computation to update $[(\Xz)^\top \Xz]^{-1}$, 
each iteration where a fusing/splitting event occurs requires $\mathcal{O}(np)$ time.
The complexity of each iteration where a switching event occurs is even smaller and only $\mathcal{O}(n)$
because we only need to update the event times for the indices switched by the event.
For more detail, see Appendix \ref{appendix-blockmatrix}.
Thus, Algorithm~\ref{oscar-algorithm} requires $\mathcal{O}(np^2 + (T_\text{fuse}+ T_\text{split})np + T_\text{switch} n)$ time where
$T_\text{fuse}$, $T_\text{split}$ and $T_\text{switch}$ are the numbers of fusing, splitting and switching events which occur until the algorithm ends, respectively.

%the dual path algorithm in~\cite{tibshirani2011} for generalized Lasso requires the complexity of $\mathcal{O}(np^2 \max \ckakko{n, p^2} + T \max\ckakko{n^2, p^4})$
%where $T$ is the number of iterations in the algorithm.

%Thus, the computational cost of Algorithm~\ref{oscar-algorithm} can be evaluated as follows.

%Algorithm~\ref{oscar-algorithm} requires $\mathcal{O}(np^2 + (T^\text{fuse}+ T^\text{split})np + T^\text{switch} n)$ time 
%where $T^\text{fuse}$, $T^\text{split}$ and $T^\text{switch}$ are the numbers of fusing, splitting and switching events which occur until the algorithm ends, respectively.

\begin{algorithm} 
	\caption{Path algorithm for clustered Lasso and OSCAR} \label{oscar-algorithm}
	\begin{algorithmic}[1]
		\State $t \leftarrow 0$, $\eta^{(0)} \leftarrow 0$, $\beta^{(0)} \leftarrow (X^\top X)^{-1}X^\top y$
		\State Compute $\group$, $\betag$, $\Xg$, $o(\cdot)$, $[(\Xz)^\top \Xz]^{-1}$, $\frac{\D \betag}{\D \eta}$, $\frac{\D \beta}{\D \eta}$ and, for OSCAR only, $s_1,\dots,s_p$.
%		\State Compute $s_1,\dots,s_p$ for OSCAR only.
%		\State Compute the order $o(g,1),\dots,o(g,p_g)$ of indices within each group $G_g$
%                \State Compute the signs $s_i = \left\{ \begin{array}{ll}  \sign{\beta_i} & \text{if $\beta_i \neq 0$,} \\ - \sign{x_i^\top (\Xtg \betatg - y)} & \text{otherwise},\end{array} \right.$ for OSCAR only.
		\While{$\frac{\D \beta}{\D \eta}\neq $\textrm{\boldmath $0$}$_p$} 
		\State Compute the times $\Delta_{g}^\text{fuse}, \Delta_{g, k}^\text{split}$ and $\Delta_{k}^\text{switch}$ of fusing, splitting and switching events.
%		\State Compute the times $\Delta_{g}^\text{fuse}$ of fusing events by \eqref{clusteredlasso-leavingtime} for clustered Lasso and \eqref{oscar-leavingtime} for OSCAR.
%		\State Compute the splitting time $l(\eta), l^0(\eta)$ which is given by \eqref{oscar-leavingtime}.
%		\State Compute the swapping time $s(\eta), s^0(\eta)$ which is given by \eqref{oscar-swappingtime}, \eqref{oscar-swappingtime0}.
		\If{$\Delta_{g}^\text{fuse}$ is minimum}
		\State $\eta^{(t+1)} \leftarrow \eta^{(t)} + \Delta_{g}^\text{fuse}$, $\beta^{(t+1)} \leftarrow \beta^{(t)} + \Delta_{g}^\text{fuse} \frac{\D \beta}{\D \eta}$.%, and update the relevant variables.
		\State Fuse $G_g$ and $G_{g+1}$, and update $\group$, $\betag$, $\Xg$, $o(\cdot)$, $[(\Xz)^\top \Xz]^{-1}$, $\frac{\D \betag}{\D \eta}$ and $\frac{\D \beta}{\D \eta}$. %The coefficients $\beta$ for two groups is sorted in descending order of %$s_i X_{\cdot, i}^\top\kakko{X^\group \betag - y}$
		\ElsIf{$\Delta_{k}^\text{switch}$ is minimum}
		\State $\eta^{(t+1)} \leftarrow \eta^{(t)} + \Delta_{k}^\text{switch}$, $\beta^{(t+1)} \leftarrow \beta^{(t)} + \Delta_{k}^\text{switch} \frac{\D \beta}{\D \eta}$.
		\If{$k=0$ (for OSCAR only)}
		\State Switch the sign of $s_{o(1)}$. % $s^{(t+1)}_{o(0,1)} \leftarrow -s^{(t)}_{o(0,1)}$.
                \Else 
                \State Switch the indices assigned to $o(k)$ and $o(k+1)$.  %Switch the order as $o^{(t+1)}(g,k) \leftarrow o^{(t)}(g,k+1)$, $o^{(t+1)}(g,k+1) \leftarrow o^{(t)}(g,k)$.  
                \EndIf
%		\State the familuy of group sets $\group^{(t+1)}$, when this two groups are merged.
%		\State Compute %$X^{\group^{(t+1)}}$ from $X^{\group^{(t)}}$.
%		\State Compute %$\dkakko{\kakko{X^{\group^{(t+1)}}_{\cdot, -z}}^\top X^{\group^{(t+1)}}_{\cdot, -z}}^{-1}$ from $\dkakko{\kakko{X^{\group^{(t)}}_{\cdot, -z}}^\top X^{\group^{(t)}}_{\cdot, -z}}^{-1}$.
%		\State $t\leftarrow t+1$
		\Else %If{$\Delta_{g, k}^\text{split}$ is minimum}
		\State $\eta^{(t+1)} \leftarrow \eta^{(t)} + \Delta_{g, k}^\text{split}$, $\beta^{(t+1)} \leftarrow \beta^{(t)} + \Delta_{g, k}^\text{split} \frac{\D \beta}{\D \eta}$.
		\State Split $G_g$, and update $\group$, $\betag$, $\Xg$, $o(\cdot)$, $[(\Xz)^\top \Xz]^{-1}$, $\frac{\D \betag}{\D \eta}$ and $\frac{\D \beta}{\D \eta}$. %compute $\group^{(t+1)}$, $\beta^{\group^{(t+1)}}$, $X^{\group^{(t+1)}}$, $[(X^{\group^{(t+1)}}_{-0})^\top X^{\group^{(t+1)}}_{-0}]^{-1}$ and $\frac{\D \beta^{\group^{(t+1)}}}{\D \eta}$.
%		\State Compute %$X^{\group^{(t+1)}}$ from $X^{\group^{(t)}}$.
%		\State Compute %$\dkakko{\kakko{X^{\group^{(t+1)}}_{\cdot, -z}}^\top X^{\group^{(t+1)}}_{\cdot, -z}}^{-1}$ from $\dkakko{\kakko{X^{\group^{(t)}}_{\cdot, -z}}^\top X^{\group^{(t)}}_{\cdot, -z}}^{-1}$.
%		\State $t\leftarrow t+1$
		\EndIf
		\State $t\leftarrow t+1$
		\EndWhile
	\end{algorithmic}
\end{algorithm}

\section{Numerical Experiment}
In this section, we evaluate the processing time and accuracy of our path algorithms through synthetic data and real data.
All the experiments are conducted on a Windows 10 64-bit machine with Intel i7-8665U CPU at 1.90GHz and 16GB of RAM.

First, we compare the processing time and the number of iterations on synthetic datasets 
between our algorithms implemented in R and the dual path algorithm (DPA)~\cite{tibshirani2011,arnold2016efficient} 
\footnote{We use the DPA in R package `genlasso'\url{https://cran.r-project.org/web/packages/genlasso/}.}. %Since the pairwise $L_{\infty}$ penalty can be converted into $L_1$ penalty as $\max \ckakko{\abs{\beta_j}, \abs{\beta_k}} = (\abs{\beta_j- \beta_k} + \abs{\beta_j + \beta_k})/2$, the dual path algorithm for the generalized Lasso can be applied to both clustered Lasso and OSCAR.
The synthetic datasets are generated from the model $y=X\beta + e$ where $e \sim N(0, I_n)$.
The covariates in $X$ are generated as independent and identical standard normal variables.
The true coefficients are given by $\beta = [\theta^\top,\theta^\top,-\theta^\top,-\theta^\top,$\textrm{\boldmath $0$}$_{0.2p}^\top]^\top \in \mathbb{R}^p$ where $\theta \sim N(0, I_{0.2p})$. %by generating $\theta \sim N(0, I_{0.2p})$ and duplicating it as $\beta = [\theta^\top,\theta^\top,-\theta^\top,-\theta^\top,$\textrm{\boldmath $0$}$_{0.2p}^\top]^\top \in \mathbb{R}^p$.
We set four levels of the problem size $[n,p]\in \{[20, 10], [60, 30], [100, 50], [200, 100]\}$
and two directions of tuning parameters $[\bar{\lambda}_{1}, \bar{\lambda}_{2}]\in \{[0, 1], [1, 1]\}$ in the path algorithms for clustered Lasso and OSCAR, respectively. % and simulate 10 datasets for each. %We set direction of tuning parameters $[\bar{\lambda}_{1}, \bar{\lambda}_{2}]$ at $[0,1]$ and $[1,1]$ for clustered Lasso and OSCAR, respectively. 
Figure~\ref{fig_runtime} shows the average running time over 10 simulated datasets for each case. %size $[n,p]$ and direction $[\bar{\lambda}_{1}, \bar{\lambda}_{2}]$. 
As the problem size gets larger,
our algorithms become much faster than the dual path algorithm~\cite{tibshirani2011,arnold2016efficient}. %which requires the complexity of $\mathcal{O}(np^2 \max \ckakko{n, p^2} + T \max\ckakko{n^2, p^4})$ where $T$ is the number of iteration.
When we count the number of iterations as in Figure~\ref{fig_iteration}, it increases rapidly with the problem size.
The number of iterations in our method is approximately doubled by the switching events, 
but still less than that in the DPA which includes events occurred only in the dual problem. %The number of iterations in OSCAR is about twice of that in clustered Lasso.

\begin{figure}[htbp]
 \begin{minipage}{0.5\hsize}
  \begin{center}
	\includegraphics[width=34mm]{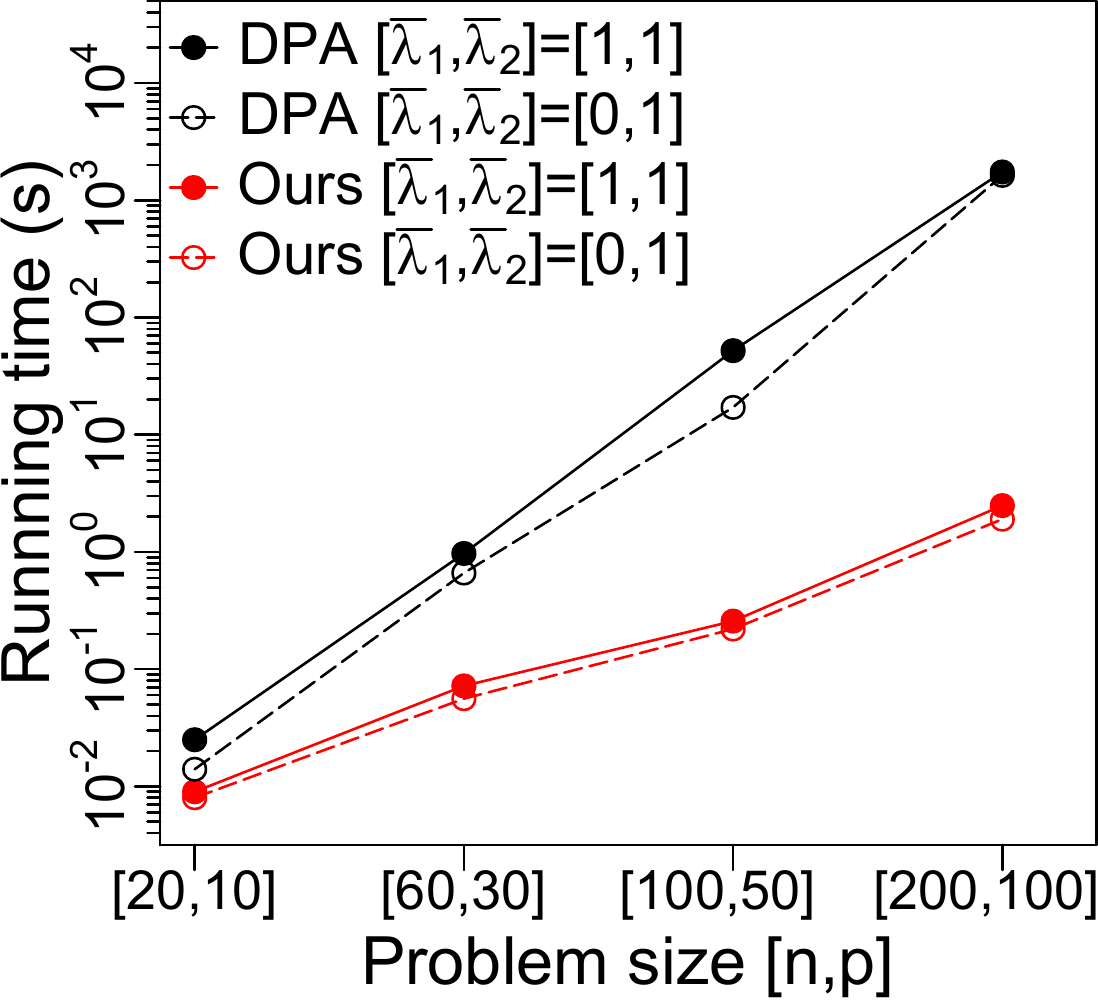}
	\includegraphics[width=34mm]{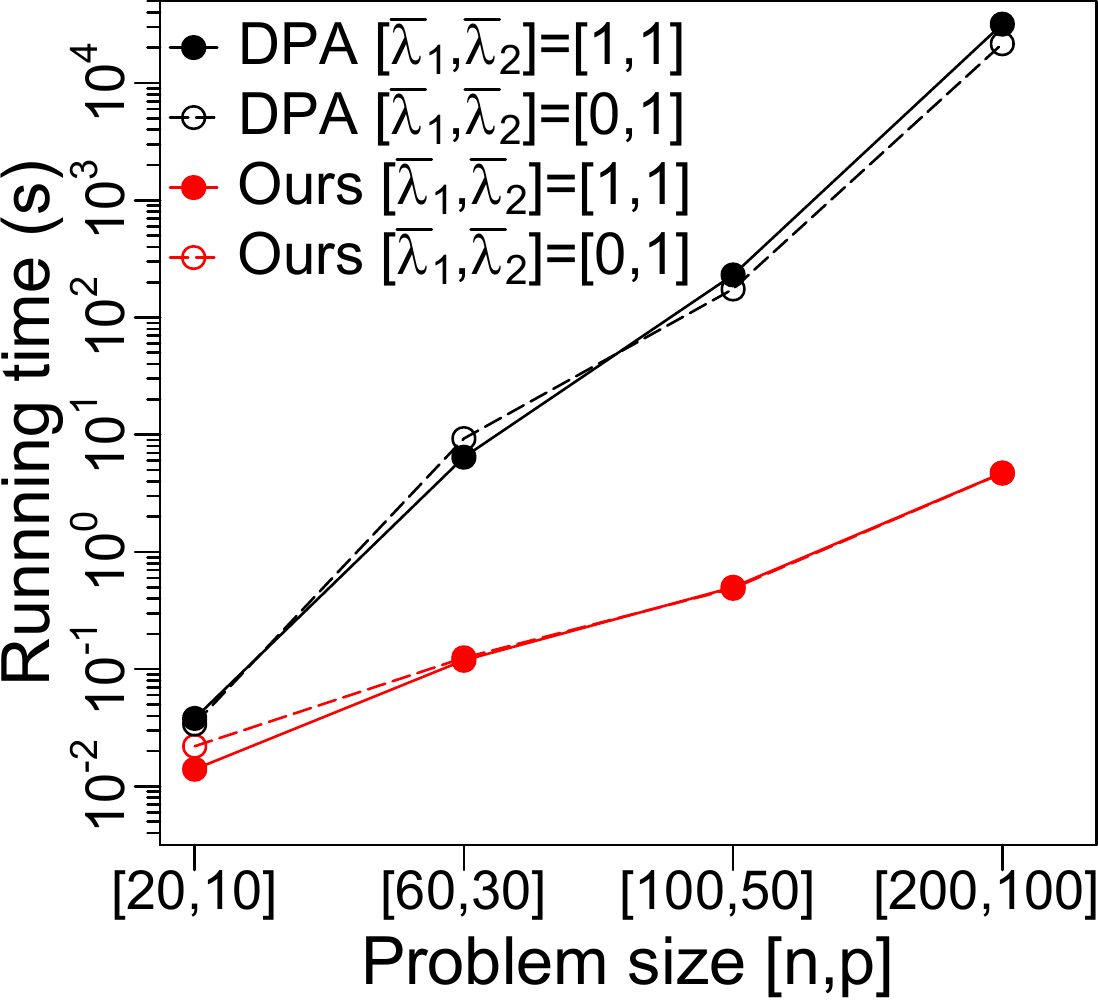}\\
  \end{center}
\hspace{0.3cm} (a) Clustered Lasso \hspace{1cm}(b) OSCAR 
  \caption{Running time on synthetic data}
  \label{fig_runtime}
 \end{minipage}
 \begin{minipage}{0.5\hsize}
  \begin{center}
	\includegraphics[width=34mm]{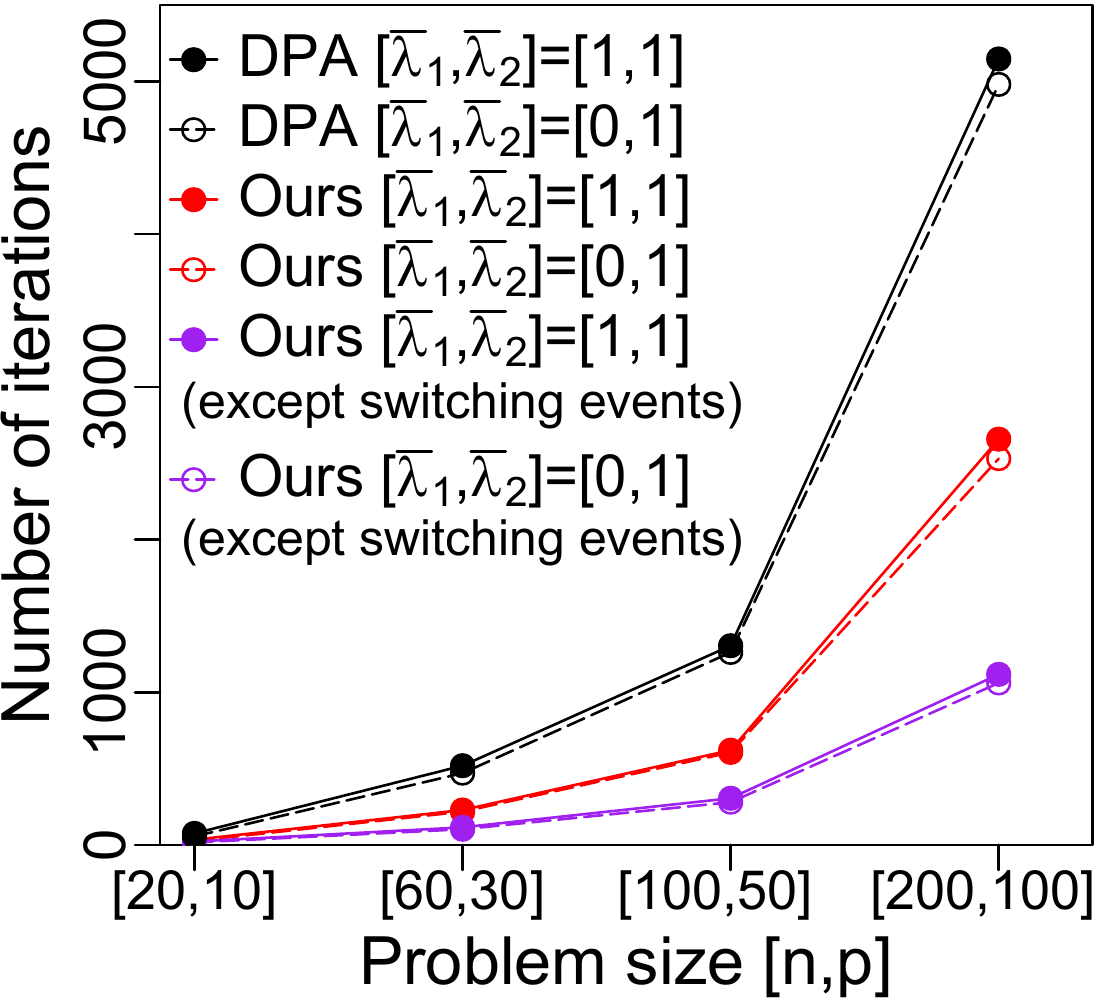}
	\includegraphics[width=34mm]{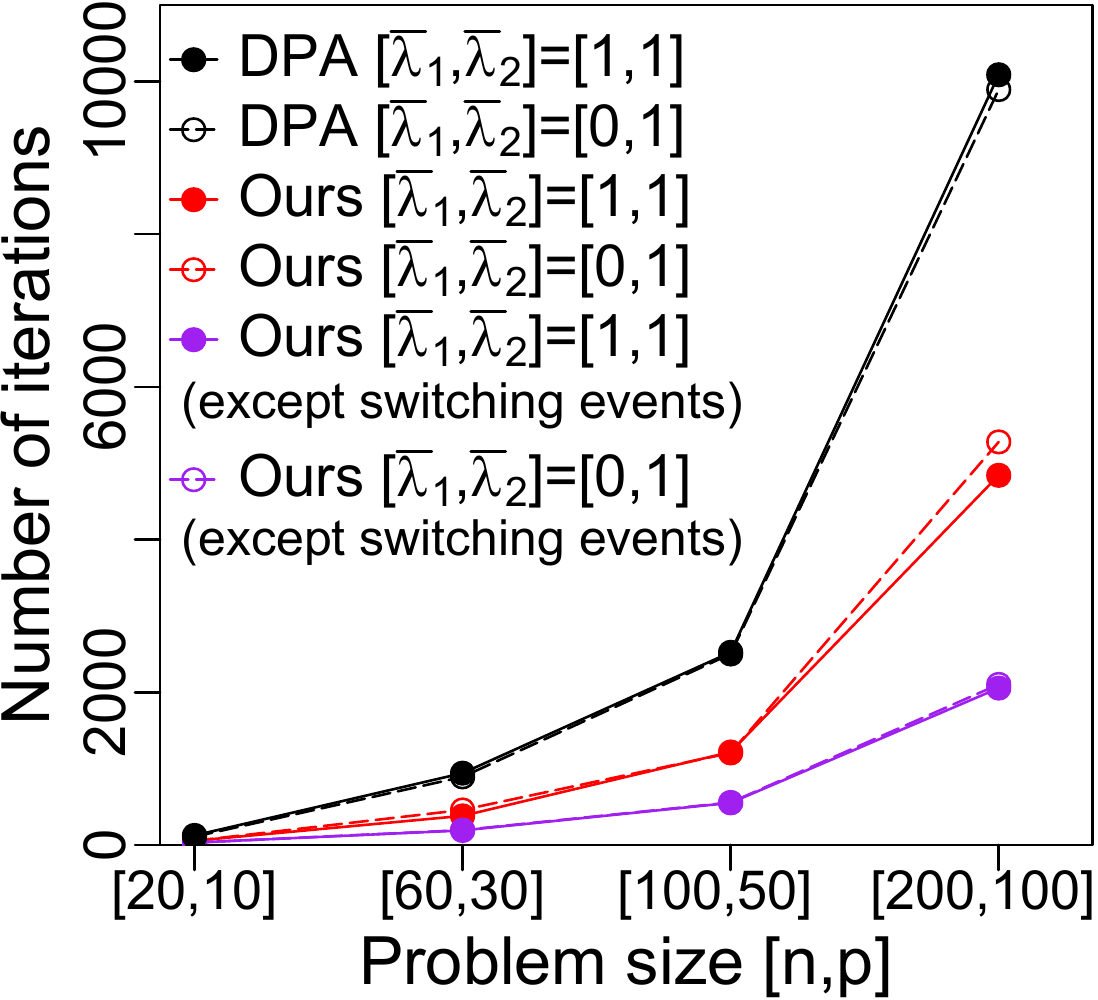}\\
  \end{center}
 \hspace{0.3cm}(a) Clustered Lasso \hspace{1cm}(b) OSCAR 
  \caption{Number of iterations on synthetic data}
  \label{fig_iteration}
 \end{minipage}
\end{figure}

%the dual path algorithm in~\cite{tibshirani2011} for generalized Lasso requires the complexity of $\mathcal{O}(np^2 \max \ckakko{n, p^2} + T \max\ckakko{n^2, p^4})$
%When we count the number of iteration as in Figure \ref{figure_experiment2}, it increase rapidly by the problem size.

 %We also evaluate the processing time per iteration and prediction error by conducting experiments
We also conduct experiments on real datasets; {\it splice} dataset from the LIBSVM data~\cite{LIBSVM2011}, %\footnote{\url{https://www.csie.ntu.edu.tw/~cjlin/libsvmtools/datasets/}},
{\it optdigits} dataset from the UCI data~\cite{Dua:2019}
and {\it brvehins2} dataset\footnote{This dataset is available in R package `CASdatasets' from \url{http://cas.uqam.ca/}.} of automobile insurance claims in Brazil. %Optdigits and YPMSD datasets contain training data and test data whose sizes are shown in Table \ref{table}, respectively.
Each dataset includes training and test data whose sizes are shown in Table \ref{table}.
For brvehins2 data, we calculate the mean amount of robbery claims in each policy as a response variable,
set 341 dummy variables for the cities where 10+ robbery claims occurred as predictors
and divide their records into training and test data evenly as $n_{\mbox{train}} = n_{\mbox{test}}$.  %as shown in Table \ref{table}. %For the UCI datasets, we use training and test data as prepared in~\cite{uci}.
For each dataset, we run 5-fold cross validation (CV) to tune $\eta$ 
and select $\bar{\lambda}_{1}$ from $\{0,0.5,1,2 \}$ while $\bar{\lambda}_{2}$ is fixed at 1.
We compare two tuning methods for $\eta$; One is a path-based search from all the event times in 5 entire solution paths of CV trials
and the other is a grid search from 100 grid points $\eta = 10^{-\frac{4i}{99}}\eta_{\max}$ $(i=0,\dots,99)$
where $\eta_{\max}$ is the terminal point of the solution path.
The solution paths are obtained by our methods implemented in Matlab 
and the solution for each grid point is given by the accelerated proximal gradient (APG) algorithms
for clustered Lasso~\cite{lin2019efficient} and OSCAR~\cite{bogdan2015}
\footnote{Its Matlab code is available at \url{http://statweb.stanford.edu/~candes/software/SortedL1/}.}.
In Table~\ref{table}, we evaluate the CV errors and test errors by the mean squared error (MSE).
Our path-based tuning of $\eta$ performs slightly better than or equally to the grid search in CV errors. % but not always in the test errors.
The test errors are also slightly different between them.
The number of nonzero groups (gnnz) in $\group(\eta)$ with $\eta$ selected by the path and grid search is also shown in Table \ref{table}.
The gnnz is sensitive to the value of $\eta$ and also differs between the path and grid search.
In Table~\ref{timeperiter}, we compare the running time per iteration with respect to event types and that per grid point in the grid search with the APG algorithms.
Our path algorithms can update solution by path events much faster than the APG algorithms.
As described in the previous section, a switching event takes much shorter time than a fusing/splitting event.

\begin{table}[htbp]
  \caption{Cross validation errors, test errors and number of nonzero groups (gnnz).}
  \label{table}
  \centering
  \begin{tabular}{ccrrrrrr}
    \toprule
    Dataset and size & Search & \multicolumn{3}{c}{Clustered Lasso}  & \multicolumn{3}{c}{OSCAR} \\
    \cmidrule(r){3-8}
    $[n_{\mbox{train}},n_{\mbox{test}},p]$ & points & CV MSE & test MSE & gnnz & CV MSE & test MSE & gnnz \\
    \midrule
    splice           & Path & 0.5734 & 0.4802 & 40 & 0.5725 & 0.4823 & 39 \\
    $[1000,2175,60]$ & Grid & 0.5734 & 0.4800 & 41 & 0.5726 & 0.4826 & 37 \\
    \midrule
    optdigits        & Path & 3.8267 & 3.7783 & 45 & 3.8265 & 3.7775 & 46 \\
    $[3823,1797,61]$ & Grid & 3.8267 & 3.7781 & 44 & 3.8265 & 3.7779 & 46 \\
    \midrule
    brvehins2           & Path & 1.5473 & 1.5124 & 153 & 1.5474 & 1.5122 & 151 \\
    $[15832,15832,341]$ & Grid & 1.5474 & 1.5125 & 158 & 1.5474 & 1.5121 & 145 \\
    \bottomrule
  \end{tabular}
\end{table}

\begin{table}[htbp]
  \caption{Running time~(s) per iteration of each event type in our algorithm and per grid point in APG.}
  \label{timeperiter}
  \centering
  \begin{tabular}{ccccrcccr}
    \toprule
    & \multicolumn{4}{c}{Clustered Lasso}  & \multicolumn{4}{c}{OSCAR} \\
    \cmidrule(r){2-9}
    Dataset & fuse & split & switch & \multicolumn{1}{c}{APG} & fuse & split & switch & \multicolumn{1}{c}{APG} \\
    \midrule
    splice     & 0.0006 & 0.0007 & 0.0001 & 0.0051 & 0.0006 & 0.0006 & 0.0001 & 0.0054 \\%    splice     & 0.00060 & 0.00067 & 0.00012 & 0.00505 & 0.00056 & 0.00057 & 0.00013 & 0.00535 \\%    splice     & $5.63\times 10^{-4}$ & $5.71\times 10^{-4}$ & $1.35\times 10^{-4}$ & $5.35\times 10^{-3}$     & $5.63\times 10^{-4}$ & $5.71\times 10^{-4}$ & $1.35\times 10^{-4}$ & $5.35\times 10^{-3}$  \\
    optdigits  & 0.0030 & 0.0031 & 0.0003 & 0.2503 & 0.0039 & 0.0040 & 0.0003 & 0.1787 \\
    brvehins2  & 0.0739 & 0.0725 & 0.0062 & 5.1822 & 0.0892 & 0.0873 & 0.0047 & 5.4041 \\
    \bottomrule
  \end{tabular}
\end{table}

\section{Conclusion}
We proposed efficient path algorithms for clustered Lasso and OSCAR.
For both problems, there are only two types of events that make change-points in solution paths, named fusing and splitting events.
By using symmetry of regularization terms, we derived simple conditions to monitor violation of optimal conditions which causes a split.
Especially, we showed that a group can be split only along a certain order of indices determined by the first derivative of the square loss. %$\frac{1}{2} \nkakko{y-X\beta}^2$.
Our approach may be extended to other sparse regularization such as SLOPE~\cite{bogdan2015}, whose penalty terms have a similar symmetric structure.
Numerical experiments showed that our algorithms are much faster than the existing methods.
Though our algorithms require enormous iterations for large problems to obtain entire solution paths,
they can be modified to make a partial solution path within an arbitrary interval of $\eta$, which may be determined by a coarse grid search
with fast solvers~\cite{zhong2012efficient, lin2019efficient, luo2019efficient}.

\section{Broader Impact}
Clustered Lasso and OSCAR are feature clustering and selection methods that can be used as powerful tools for dimensionality reduction of feature space. %than simple feature selection methods like Lasso. %Feature clustering and selection methods by clustered Lasso and OSCAR are stronger ways of dimensionality reduction than simple feature selection methods like Lasso.
Our path algorithms can provide fine tuning of regularization parameters for clustered Lasso and OSCAR in reasonable time.
We illustrated some applications of our algorithms to DNA microarray analysis, image analysis and actuarial science in this paper.
However, we should avoid abusing such methods to irrelevant features, which may cause misinterpretation of grouped features.

\appendix
\section{Proofs of Theorem 1 and Corollary 1} \label{appendix-proof1}
This section provide proofs of Theorem \ref{theorem1} and Corollary \ref{corollary1}.
To prove that theorem, we prove a lemma extended from Theorem 2 in \cite{hoefling2010path}.
Theorem 2 in \cite{hoefling2010path} states that subgradient equations in a fused Lasso signal approximator (i.e. $X$ is an identical matrix)
without Lasso terms (i.e. $\lambda_1=0$) can be checked through a maximum flow problem on the underlying graph whose edges correspond to the pairwise fused Lasso penalties. 
Here we extend a part of the statements in the theorem into the weighted fused Lasso problem with weighted Lasso terms by introducing the {\it flow network} $G=\{(V,E),c,r,s\}$ defined as follows:
\begin{description}
	\item[Vertices:] 
        Define the vertices by $V = \{0,1,\dots,m\} \cup \{r,s\}$ where $r$ and $s$ are called the {\it source} and the {\it sink}, respectively.        %We add an aritificial source node $r$, sink node $s$, and zero node $0$.        to the subgraph $\mathcal{G}_{G_z}$, i.e. $\tilde{V}=G_z \cup \ckakko{r, s, 0}$.
	\item[Edges:] 
        Define the edges by $E = \ckakko{(i,j); i,j\in V, i\neq j } \setminus \ckakko{(r,s),(s,r)}$,
        that is, all the pairs of vertices except for the pair of the source and sink are linked to each other.
	\item[Capacities:] Define the capacities on the edges in $E$ by
	\begin{align*}
	c(r,i) &= f_i^{-},\; c(i,r) = 0, & i = 0,\dots,m, \\
	c(i,s) &= f_i^{+},\; c(s,i) = 0, & i = 0,\dots,m, \\
	c(i,0) &= c(0,i) = \lambda_{i0}, & i = 1,\dots,m, \\
	c(i,j) &= c(j,i) = \lambda_{ij}, & 1 \le i < j \le m, 
	\end{align*}
        where $f_0 = -\sum_{j=1}^m f_j$, $f_i^{-} = \max\ckakko{-f_i,0}$ and $f_i^{+} = \max\ckakko{f_i,0}$.
        Note that we have $f_i = f_i^{+}-f_i^{-}$ and hence $\sum_{j=0}^m f_j^{+} - \sum_{j=0}^m f_j^{-} = \sum_{j=0}^m f_j = 0$.
\end{description}
A {\it flow} on the flow network $G=\{(V,E),c,r,s\}$ is a set of values $\tau = \{ \tau_{ij};(i,j)\in E \}$ satisfying the following three properties:
\begin{description}
	\item[Skew symmetry:] $\tau_{ij} = -\tau_{ji}$ for $(i,j) \in E$,
	\item[Capacity constraint:] $\tau_{ij} \le c(i,j)$ for $(i,j) \in E$,
	\item[Flow conservation:] $\sum_{j; (i,j) \in E} \tau_{ij} = 0 $ for $i \in V\setminus \{r,s\}$.
\end{description}
The {\it value} $\abs{\tau}$ of a flow $\tau$ through $r$ to $s$ is the net flow out of the source $r$ or that into the sink $s$ formulated as follows:
\begin{equation*}
\abs{\tau} = \sum_{i; (r,i) \in E} \tau_{ri} = \sum_{i; (i,s) \in E} \tau_{is},
\end{equation*}
where the last equality holds from the flow conservation.
The {\it maximum flow problem} on $G=\{(V,E),c,r,s\}$ is a problem to find the flow that attains the maximum value $\max_\tau \abs{\tau}$, called the {\it maximum flow}.
For an overview of the theory of the maximum flow problems, see e.g.~\cite{tarjan1983network}.

Then, we obtain the following lemma relating the condition \eqref{lemma1-subgradequation} to the maximum flow problem on $G=\{(V,E),c,r,s\}$. 
%In this ways, this problem can be regarded as the network flow problem.  When the flow of this problem reach the maximum capacities, except for the capacities which is linked to the source node or sink node, then the solution for $(\lambda_{1} + \Delta \lambda_{1}, \lambda_{2} + \Delta \lambda_{2})$ doesn't exitst, i.e. the solution path of this problem will split. In the previous setion, we assumed that the capacites of the edges $E(G_g)$ are the same weighs. But in this section, the capacities of the edges linked to the zero node are different from the capacities of the edges for $G_z$. In the next section, we will propose the method of detecting the split event for this maximum flow problem effificiently 
\begin{lemma} \label{lemma11}
	There exist $\tau_{i0} \in \dkakko{-\lambda_{i0}, \lambda_{i0}}$ and $\tau_{ij} = -\tau_{ji} \in \dkakko{-\lambda_{ij}, \lambda_{ij}}$ such that
	\begin{equation} \label{lemma11-subgradequation}
	f_i + \tau_{i0} + \sum_{j \in \{1,\dots,m\} \setminus\ckakko{i}}\tau_{ij} =0, \qquad i=1,\dots,m,
	\end{equation}
	if and only if the maximum value of a flow through $r$ to $s$ on $G=\{(V,E),c,r,s\}$ is $\sum_{i=0}^m f_i^{+} = \sum_{i=0}^m f_i^{-}$.
\end{lemma}
\begin{proof}[proof]
Consider a flow $\tau = \{ \tau_{ij} ; (i,j)\in E\}$ on the flow network $G=\{(V,E),c,r,s\}$ defined above. % such that
%\begin{description}
%	\item[Skew symmetry:] $\tau_{ij} = -\tau_{ji}$ for $(i,j) \in E$,
%	\item[Capacity constraint:] $\tau_{ij} \le c(i,j)$ for $(i,j) \in E$,
%	\item[Conservation of flows:] $\sum_{j; (i,j) \in E} \tau_{ij} = 0 $ for $i \in V\setminus \{r,s\}$.
%\end{description}
%From the skew symmetry and the capacity constraint, we obtain $\tau_{i0} \in \dkakko{-\lambda_{i0}, \lambda_{i0}}$ and $\tau_{ij} = -\tau_{ji} \in \dkakko{-\lambda_{ij}, \lambda_{ij}}$.
From the capacity constraint, the value $\abs{\tau} = \sum_{i=0}^m \tau_{ri} = \sum_{i=0}^m \tau_{is}$ of the flow 
cannot be more than $\sum_{i=0}^m c(r,i) = \sum_{i=0}^m c(i,s) = \sum_{i=0}^m f_i^{+} = \sum_{i=0}^m f_i^{-}$.
Therefore, this value of the flow is attained if and only if $\tau_{ri} = c(r,i) = f_i^{-}, \tau_{is} = c(i,s) = f_i^{+}$ for all $i=0,\dots,m$,
and hence 
\begin{equation} \label{consflow}
\sum_{j\in V\setminus \{i\}} \tau_{ij} = f_i + \tau_{i0} + \sum_{j \in \{1,\dots,m\} \setminus\ckakko{i}}\tau_{ij} = 0, \qquad i = 1,\dots,m,
\end{equation}
from the flow conservation where $\tau_{i0} \in \dkakko{-\lambda_{i0}, \lambda_{i0}}$ and $\tau_{ij} = -\tau_{ji} \in \dkakko{-\lambda_{ij}, \lambda_{ij}}$ 
from the skew symmetry and the capacity constraint.
Furthermore, the flow conservation at the vertex zero $\sum_{i\in V\setminus \{0\}} \tau_{0i} = f_0 + \sum_{i=1}^m \tau_{0i} = - \sum_{i=1}^m f_i - \sum_{i=1}^m \tau_{i0} = 0$
follows from the sum of \eqref{consflow} over $i \in \{1,\dots,m\}$, completing the proof.
\end{proof}
Note that the equation \eqref{lemma11-subgradequation} appears in subgradient equation within a fused group for the weighted clustered Lasso problem
with penalty terms $\sum_{i=1}^m \lambda_{i0}|\beta_i| + \sum_{1 \le i < j \le m} \lambda_{ij}|\beta_i-\beta_j| $.
We can use Lemma \ref{lemma11} to check the subgradient equation by seeking the maximum flow on the corresponding flow network.
However, in the weighted clustered Lasso problem, it is generally difficult to find when the subgradient equation is violated as the regularization parameters grow.
We can provide the explicit condition to check the violation of \eqref{lemma1-subgradequation} for only the ordinary clustered Lasso problem as in Theorem \ref{theorem1}, whose proof is provided as follows.
%Then, applying the minimum cut-maximum flow theorem, we can prove Theorem \ref{theorem1} as follows.
\begin{proof}[proof of Theorem \ref{theorem1}]
%Since $f_1 \ge f_2 \ge \cdots \ge f_m$, let $l$ denote the last index of non-negative $f_i$, i.e. $f_l \ge 0 > f_{l+1}$.Then, $(i,s)\in E$ for $i=1,\dots l$ and $(
Here we consider the {\it minimum cut problem} on the flow network $G=\{(V,E),c,r,s\}$ where $\lambda_{i0} = \lambda_1$ and $\lambda_{ij}=\lambda_2$ for all $i\neq j \in \{1,\dots,m\}$.
A {\it cut} of the graph is a partition of the vertex set $V$ into two parts $V_r$ and $V_s = V \setminus V_r$ such that $r\in V_r$ and $s\in V_s$.
Then, the {\it capacity} of the cut is defined by the sum of capacities on edges $(i,j)\in E$ such that $i\in V_r$ and $j\in V_s$.
In this proof, we use the max-flow min-cut theorem (see e.g.~\cite{tarjan1983network}), that states 
the maximum value of a flow through $r$ to $s$ on $G=\{(V,E),c,r,s\}$ is equal to the minimum capacity of a cut of the same graph.
Thus, from that theorem and Lemma \ref{lemma11}, it suffices to prove that the minimum capacity of a cut of $G=\{(V,E),c,r,s\}$
is just $\sum_{i=0}^m f_i^{+} = \sum_{i=0}^m f_i^{-}$ if and only if \eqref{lowerbound} and \eqref{upperbound} hold.

For $V_r = \{r\},V_s = V\setminus \{r\}$ and $V_r=V\setminus \{s\} ,V_s=\{s\}$, 
the capacity of the cut attains $\sum_{i=0}^m f_i^{+} = \sum_{i=0}^m f_i^{-}$, which is supposed to be the minimum capacity of a cut.
For the other cuts, let $k$ and $m-k$ denote the cardinality of $\tilde{V}_s = V_s \cap \{1,\dots,m\}$ and $\tilde{V}_r = V_r \cap \{1,\dots,m\}$, respectively.
Then, the capacity of the cut $C(V_r,V_s)$ is obtained by, if $0\in V_s$, 
\begin{align} 
C(V_r,V_s) &= \sum_{i\in V_r, j\in V_s, (i,j)\in E} c(i,j) \nonumber \\
&= \sum_{i \in \tilde{V}_r} c(i,s) + \sum_{i \in \tilde{V}_s} c(r,i) + c(r,0) +
\sum_{i \in \tilde{V}_r} c(i,0) + \sum_{i \in \tilde{V}_r, j\in \tilde{V}_s} c(i,j)\nonumber  \\
&= \sum_{i \in \tilde{V}_r} f_i^{+} + \sum_{i \in \tilde{V}_s} f_i^{-} + f_0^{-} +
\sum_{i \in \tilde{V}_r} \lambda_1 + \sum_{i \in \tilde{V}_r, j\in \tilde{V}_s} \lambda_2\nonumber  \\
&= \sum_{i \in \tilde{V}_r} f_i^{+} + \sum_{i \in \tilde{V}_s} f_i^{-} + f_0^{-} + (m-k)\lambda_1 + k(m-k)\lambda_2, \label{cut-vszero}
\end{align}
and otherwise
\begin{align} 
C(V_r,V_s) &= \sum_{i \in \tilde{V}_r} c(i,s) + \sum_{i \in \tilde{V}_s} c(r,i) + c(0,s) + 
\sum_{i \in \tilde{V}_s} c(0,i) + \sum_{i \in \tilde{V}_r, j\in \tilde{V}_s} c(i,j)\nonumber  \\
&= \sum_{i \in \tilde{V}_r} f_i^{+} + \sum_{i \in \tilde{V}_s} f_i^{-} + f_0^{+} + 
\sum_{i \in \tilde{V}_s} \lambda_1 + \sum_{i \in \tilde{V}_r, j\in \tilde{V}_s} \lambda_2\nonumber  \\
&= \sum_{i \in \tilde{V}_r} f_i^{+} + \sum_{i \in \tilde{V}_s} f_i^{-} + f_0^{+} + k\lambda_1 + k(m-k)\lambda_2. \label{cut-vrzero}
\end{align}
Therefore, since $f_1 \ge f_2 \ge \cdots \ge f_m$, $C(V_r,V_s)$ is bounded by
\begin{equation} \label{cut-vszerobound}
C(V_r,V_s) \ge \sum_{i=k+1}^m f_i^{+} + \sum_{i=1}^k f_i^{-} + f_0^{-} + (m-k)\lambda_1 + k(m-k)\lambda_2,
\end{equation}
from \eqref{cut-vszero} and
\begin{equation}  \label{cut-vrzerobound}
C(V_r,V_s) \ge \sum_{i=k+1}^m f_i^{+} + \sum_{i=1}^k f_i^{-} + f_0^{+} + k\lambda_1 + k(m-k)\lambda_2,
\end{equation}
from \eqref{cut-vrzero}.
The equalities in both \eqref{cut-vszerobound} and \eqref{cut-vrzerobound} hold when $\tilde{V}_s  = \{1,\dots,k\}$ and $\tilde{V}_r = \{k+1,\dots,m\}$.
Thus, $C(V_r,V_s) \ge \sum_{i=0}^m f_i^{+} = \sum_{i=0}^m f_i^{-}$ for any cut if and only if
\eqref{lowerbound} and \eqref{upperbound} hold, completing the proof.
\end{proof}
Corollary \ref{corollary1} is derived from Theorem \ref{theorem1} as follows.
\begin{proof}[proof of Corollary \ref{corollary1}]
By taking $\lambda_1=0$ in Theorem \ref{theorem1}, the equation \eqref{lemma1-subgradequation} reduces to \eqref{coro1}
and the inequality \eqref{lowerbound} reduces to \eqref{coro2}.
Furthermore, combining the inequalities \eqref{lowerbound} for $k=m$ and \eqref{upperbound} for $k=0$, we obtain $\sum_{j=1}^{m} f_j = 0$.
Then, the inequality \eqref{upperbound} for $k=1,\dots,m-1$ holds because 
\begin{equation*}
\sum_{j=k+1}^{m} f_j = \sum_{j=1}^{m} f_j - \sum_{j=1}^{k} f_j = - \sum_{j=1}^{k} f_j \ge -\lambda_2 k(m-k),
\end{equation*}
from \eqref{coro2}, completing the proof.
\end{proof}

\section{Proof of Theorem 2} \label{appendix-proof2}
This section provide a proof of Theorem \ref{theorem2}.
For the proof of Theorem 2, we introduce the following flow network $G=\{(V,E),c,r,s\}$ defined differently from that in Theorem \ref{theorem1}. %as follows. %The proof of this lemma consists of two parts. The first part is to derive the equivalent condition via the maximum flow problem defined as follows.
\begin{description}
	\item[Vertices:] 
        Define the vertices by $V = U \cup W \cup \{r,s\}$ where $U = \{u_1,\dots,u_m\}$ and $W=\{ w_{ij}; 1\le i < j \le m \}$.        %We add an aritificial source node $r$, sink node $s$, and zero node $0$.        to the subgraph $\mathcal{G}_{G_z}$, i.e. $\tilde{V}=G_z \cup \ckakko{r, s, 0}$.
	\item[Edges:] 
        Define the edges by $E = \ckakko{(r,u_i); i =1,\dots,m} \cup \ckakko{(u_i,s); i =1,\dots,m} \cup \ckakko{(u_i,w_{jk}); i\in \{j,k\},1\le j < k \le m } \cup \ckakko{(w_{ij},s); 1\le i < j \le m }$.
	\item[Capacities:] Define the capacities on the edges in $E$ by
	\begin{align*}
	c(r,u_i) &= \abs{f_i}, & i=1,\dots,m, \\
	c(u_i,s) &= \lambda_{i0}, & i=1,\dots,m, \\
	c(u_i,w_{ij}) &= c(u_j,w_{ij}) = \lambda_{ij}, & 1\le i < j \le m, \\
	c(w_{ij},s) &= \lambda_{ij}, & 1\le i < j \le m,
	\end{align*}
        and, for convenience, $c(v,v^\prime) = 0$ if $(v,v^\prime) \notin E$ but $(v^\prime,v) \in E$. %zero for the adverse direction of each edge above.
\end{description}
Then, we obtain the following lemma relating the condition \eqref{lemma2-subgradequation} to the maximum flow problem on $G=\{(V,E),c,r,s\}$. 
%In this ways, this problem can be regarded as the network flow problem.  When the flow of this problem reach the maximum capacities, except for the capacities which is linked to the source node or sink node, then the solution for $(\lambda_{1} + \Delta \lambda_{1}, \lambda_{2} + \Delta \lambda_{2})$ doesn't exitst, i.e. the solution path of this problem will split. In the previous setion, we assumed that the capacites of the edges $E(G_g)$ are the same weighs. But in this section, the capacities of the edges linked to the zero node are different from the capacities of the edges for $G_z$. In the next section, we will propose the method of detecting the split event for this maximum flow problem effificiently 
\begin{lemma} \label{lemma21}
	There exist $\xi_{i0} \in \dkakko{-\lambda_{i0}, \lambda_{i0}}$ and $\xi_{ij},\xi_{ji} \in \dkakko{-\lambda_{ij}, \lambda_{ij}}$ such that $\abs{\xi_{ij}} + \abs{\xi_{ji}} \le \lambda_{ij}$ and 
	\begin{equation} \label{lemma21-subgradequation}
	f_i + \xi_{i0} + \sum_{j \in \{1,\dots,m\} \setminus\ckakko{i}}\xi_{ij} =0, \qquad i=1,\dots,m, %\\
%        , \quad 1\le i<j\le m,
	\end{equation}
%	\begin{align*} \label{lemma21-subgradequation}
%	f_i + \xi_{i0} + \sum_{j \in \{1,\dots,m\} \setminus\ckakko{i}}\xi_{ij} &=0, \qquad i=1,\dots,m, \\
%        \abs{\xi_{ij}} + \abs{\xi_{ji}} &\le \lambda_{ij}, \qquad 1\le i<j\le m, \nonumber
%	\end{align*}
	if the maximum value of a flow through $r$ to $s$ on $G=\{(V,E),c,r,s\}$ is $\sum_{i=1}^m \abs{f_i}$.
\end{lemma}
\begin{proof}[proof]
Let $\tau_{ij}$ and $\tau_{ji}$ denote the flows from $u_i$ and $u_j$ to $w_{ij}$, respectively, bounded by $0\le \tau_{ij}, \tau_{ji}\le \lambda_{ij}$ 
from their capacity constraints.
Then, from the flow conservation at $w_{ij}$, the flow from $w_{ij}$ to the sink $s$ must be $\tau_{ij} + \tau_{ji}$,
which is also bounded by $0\le \tau_{ij} + \tau_{ji}\le \lambda_{ij}$.
We also denote by $\tau_{i0}$ a flow from $u_i$ to $s$, bounded by $0 \le \tau_{i0} \le \lambda_{i0}$ from its capacity constraint.
From the capacity constraint, the value of a flow cannot be more than $\sum_{i=1}^m \abs{f_i}$,
which is attained if and only if the flows from the source $r$ to $u_i$ is $\abs{f_i}$ for all $u_i\in U$
and hence, from the flow conservation at $u_i\in U$, $-\abs{f_i} + \tau_{i0} + \sum_{j \in \{1,\dots,m\} \setminus\ckakko{i}}\tau_{ij} =0$.
If such a flow exist,  % $\{ \tau_{ij}; i\in \{1,\dots,m\},\; j\in \{0,\dots,m\}\setminus \{i\} \}$
\eqref{lemma21-subgradequation} is satisfied by setting $\xi_{ij} = s_i\tau_{ij}$ for $i \in \{1,\dots,m\}, j\in \{0,\dots,m\}\setminus \{i\}$ where $s_i = -\sign{f_i}$.
\end{proof}
Note that the equation \eqref{lemma21-subgradequation} appears in the subgradient equation within the fused group of zeros for the weighted OSCAR problem
with penalty terms $\sum_{i=1}^m \lambda_{i0}|\beta_i| + \sum_{1 \le i < j \le m} \lambda_{ij}\max\{|\beta_i|,|\beta_j|\} $.
We can use Lemma \ref{lemma21} to check the subgradient equation by seeking the maximum flow on the corresponding flow network.
However, in the weighted OSCAR problem, it is generally difficult to find when the subgradient equation is violated as the regularization parameters grow.
We can only provide the explicit condition to check the violation of \eqref{lemma1-subgradequation} for the ordinary OSCAR problem as in Theorem \ref{theorem2}, whose proof is provided as follows. %Then, applying the minimum cut-maximum flow theorem, we can prove Theorem 2 as follows.
\begin{proof}[proof of Theorem \ref{theorem2}]
First, it is easy to verify that \eqref{lemma2-subgradequation} implies \eqref{lemma2-condition} as follows:
\begin{equation*}
\begin{split}
\sum_{i=k+1}^{m} \abs{f_i} &\le \sum_{i=k+1}^{m}\kakko{\abs{\xi_{i0}} + \sum_{j\neq i}\abs{\xi_{ij}}} \\
&= \sum_{i=k+1}^{m} \abs{\xi_{i0}} + \sum_{k+1\le i<j\le m}(\abs{\xi_{ij}}+\abs{\xi_{ji}}) + \sum_{i=k+1}^{m} \sum_{j=1}^{k} \abs{\xi_{ij}} \\
&\le \lambda_1(m-k) + \lambda_2 \frac{(m-k)(m-k-1)}{2} + \lambda_2 k(m-k).
\end{split}
\end{equation*}
Consider the minimum cut problem on the same flow network $G=\{(V,E),c,r,s\}$ where $\lambda_{i0} = \lambda_1$ and $\lambda_{ij}=\lambda_2$ for all $i\neq j \in \{1,\dots,m\}$.
Let $V_r, V_s \subset V$ denote a cut of the graph such that $V_r \cup V_s = V$, $V_r \cap V_s = \emptyset$, $r\in V_r$, $s\in V_s$.
We also denote $U_r = V_r \cap U$, $U_s = V_s \cap U$, $W_r = V_r \cap W$ and $W_s = V_s \cap W$.
From the max-flow min-cut theorem and Lemma \ref{lemma21}, the minimum capacity of a cut of $G=\{(V,E),c,r,s\}$ has to be just $\sum_{i=1}^m \abs{f_i}$.
When $V_r = \{r\}$ and $V_s = V\setminus \{r\}$, the capacity of the cut attains $\sum_{i=1}^m \abs{f_i}$.
Thus, it suffices to prove that capacity of any other cuts of the graph cannot be less than $\sum_{i=1}^m \abs{f_i}$ 
if \eqref{lemma2-condition} holds. %as follows.

The capacity $C(V_r,V_s)$ of the cut can be decomposed as follows:
\begin{align} 
C(V_r,V_s) &= \sum_{v\in V_r, v^\prime \in V_s, (v,v^\prime)\in E} c(v,v^\prime) \nonumber \\
&= \sum_{u \in U_s} c(r,u) + \sum_{u \in U_r} c(u,s) +
\sum_{u\in U_r, w\in W_s, (u,w)\in E} c(u,w) + \sum_{w \in W_r} c(w,s). \label{cut-oscar}
\end{align}
Since $\abs{f_1}\le \abs{f_2} \le \cdots \le \abs{f_m}$, the first and second terms in \eqref{cut-oscar} are bounded by 
\begin{equation*} 
\sum_{u \in U_s} c(r,u) + \sum_{u \in U_r} c(u,s) \ge \sum_{i=1}^k \abs{f_i} + \lambda_1 (m-k),
\end{equation*}
where $k$ is the cardinality of $U_s$. To bound the rest terms in \eqref{cut-oscar}, 
let $C(V_r,V_s;w_{ij})$ denote the capacity of the cut within the edges including $w_{ij}$ formulated as follows:
\begin{equation*} 
C(V_r,V_s;w_{ij}) = \begin{cases}
c(w_{ij},s) & \text{if $w_{ij} \in W_r$}, \\
\sum_{u\in \{u_i,u_j\} \cap U_r} c(u,w_{ij}) & \text{if $w_{ij} \in W_s$}.
\end{cases}
\end{equation*}
Then, the third and fourth terms in \eqref{cut-oscar} are represented by $\sum_{u\in U_r, w\in W_s, (u,w)\in E} c(u,w) + \sum_{w \in W_r} c(w,s) = \sum_{w\in W} C(V_r,V_s;w)$. 
Furthermore, $C(V_r,V_s;w_{ij})$ can be evaluated as follows:
\begin{enumerate}
	\item If $u_i, u_j \in U_r$, we have
		\begin{equation*}
			C(V_r,V_s;w_{ij}) = \begin{cases}
				\lambda_2 & \text{if $w_{ij} \in W_r$}, \\
				2\lambda_2 & \text{if $w_{ij} \in W_s$}.
			\end{cases}
		\end{equation*}
	\item If $u_i, u_j \in U_s$, we have
		\begin{equation*}
			C(V_r,V_s;w_{ij}) = \begin{cases}
				\lambda_2 & \text{if $w_{ij} \in W_r$}, \\
				0 & \text{if $w_{ij} \in W_s$}.
			\end{cases}
		\end{equation*}
	\item Otherwise, we have $C(V_r,V_s;w_{ij}) = \lambda_2$ whichever $w_{ij} \in W_r$ or $w_{ij} \in W_s$.
\end{enumerate}
Thus, since we have $\frac{(m-k)(m-k-1)}{2}$, $\frac{k(k+1)}{2}$ and $k(m-k)$ cases for (i), (ii) and (iii), respectively,
the capacity $C(V_r,V_s)$ of the cut is bounded by
\begin{equation*} 
C(V_r,V_s) \ge \sum_{i=1}^k \abs{f_i} + \lambda_1 (m-k) + \lambda_2 \frac{(m-k)(m-k-1)}{2} + \lambda_2 k(m-k),
\end{equation*}
where the equality holds when we set, for example, $U_s = \{u_1,\dots,u_k\}$ and $W_s=\{ w_{ij}; 1\le i<j\le k \}$.
Therefore, $C(V_r,V_s) \ge \sum_{i=1}^m \abs{f_i}$ for any cut if \eqref{lemma2-condition} holds, completing the proof.
\end{proof}

\section{Complexity of the path algorithms} \label{appendix-blockmatrix}
%In this section, we discuss the complexity of our path algorithm for clustered Lasso and OSCAR.
The total computational cost of our path algorithms for clustered Lasso and OSCAR is $\mathcal{O}(np^2 + (T_\text{fuse}+ T_\text{split})np + T_\text{switch} n)$ time where
$T_\text{fuse}$, $T_\text{split}$ and $T_\text{switch}$ are the numbers of iterations in which fusing, splitting and switching events occur, respectively.
The first term $\mathcal{O}(np^2)$ is required to obtain the initial solution $\beta^{(0)} = (X^\top X)^{-1}X^\top y$.
In the following subsections, we derive the complexity per iteration for each event type.

%As for the computational cost, $\mathcal{O}(np^2)$ time is required to obtain the initial solution $\beta^{(0)} = (X^\top X)^{-1}X^\top y$.
%By using a block matrix computation to update $[(\Xz)^\top \Xz]^{-1}$, 
%each iteration where a fusing/splitting event occurs requires $\mathcal{O}(np)$ time.
%The complexity of each iteration where a switching event occurs is even smaller and only $\mathcal{O}(n)$
%because we only need to update the event times for the indices switched by the event.
%For more detail, see the supplementary material.
%Thus, Algorithm~\ref{oscar-algorithm} requires $\mathcal{O}(np^2 + (T_\text{fuse}+ T_\text{split})np + T_\text{switch} n)$ time where
%$T_\text{fuse}$, $T_\text{split}$ and $T_\text{switch}$ are the numbers of fusing, splitting and switching events which occur until the algorithm ends, respectively.

%Especially, we focus on the update of $[(\Xz)^\top \Xz]^{-1}$ where a few columns in $\Xz$ involved in a fusing/splitting event are replaced.

\subsection{Complexity per iteration for fusing/splitting events}
In this subsection, we discuss the complexity of the iteration where a fusing/splitting event occurs. %computational cost after a fusing/splitting event is selected as the occurrence event in a iteration.
More specifically, we evaluate the computational cost of updating $\group$, $\betag$, $\Xg$, $o(\cdot)$, $[(\Xz)^\top \Xz]^{-1}$, $\frac{\D \betag}{\D \eta}$ and $\frac{\D \beta}{\D \eta}$
by the fusion/split of the groups and then calculating the next timings of the events.

%The coefficients $\beta$ for two groups is sorted in descending order of %$s_i X_{\cdot, i}^\top\kakko{X^\group \betag - y}$
%from the time a fusing/splitting event is selected to the time the next event is selected.%First, we discuss the complexity of the iteration from the time a fusing/splitting event is selected to the time the next event is selected.

First, we focus on the update of $[(\Xz)^\top \Xz]^{-1}$ where a few columns in $\Xz$ are replaced by a fusing/splitting event.
When the set $\mathcal{G}$ of fused groups is changed into $\mathcal{\tilde{G}}$ by a fusing/splitting event,
let $X^{\mathcal{G}} = [X_{(1)},X_{(2)}]$ and $X^{\mathcal{\tilde{G}}} = [X_{(1)},\tilde{X}_{(2)}]$ denote
the grouped design matrices sharing some columns $X_{(1)} \in \mathbb{R}^{n\times p_{(1)}}$ 
but having different ones $X_{(2)}\in \mathbb{R}^{n\times p_{(2)}}$ and $\tilde{X}_{(2)}\in \mathbb{R}^{n\times \tilde{p}_{(2)}}$, respectively. %For simplicity, 
We can permute the columns of $X^{\mathcal{G}}$ and $X^{\mathcal{\tilde{G}}}$ to apply those notations 
and recover their original orders after the update. %in $\mathcal{O}(np_{(2)})$ time and $\mathcal{O}(n\tilde{p}_{(2)})$ time, respectively.
Note that, because at most two groups are involved in a fusing/splitting event, 
the number of columns replaced by a fusing/splitting event cannot exceed two, that is, $p_{(2)},\tilde{p}_{(2)}\le 2$.
Moreover, because we assume $n\ge p$ and $\rank{X} = p$, 
we have $n > p_{(1)}$ and the inverse of $(\Xz)^\top \Xz$ always exists. %unless all the coefficients are zero.
%we have $n > p_{(1)}$ We assume that the inverse matrix $[(\Xz)^\top \Xz]^{-1}$ exists, which implies $n \ge p_{(1)}$. %have $n \ge p > p_{(1)}$ so that the inverse matrix $[(\Xz)^\top \Xz]^{-1}$ exists.

Then, we can update $Z = [(\Xz)^\top \Xz]^{-1}$ into $\tilde{Z} = [(\Xzt)^\top \Xzt]^{-1}$ by using the following lemma:
\begin{lemma} \label{lemma3}
	Given $X_1,X_2,\tilde{X}_2$ and $Z = [(\Xz)^\top \Xz]^{-1}$ decomposed in a block matrix
\begin{equation*} %\label{blockinverse}
 Z = \begin{bmatrix} Z_{11} & Z_{12} \\ Z_{21} & Z_{22} \end{bmatrix},
\end{equation*}
        with $Z_{11} \in \mathbb{R}^{p_{(1)}\times p_{(1)}}$, $Z_{12}=Z_{21}^\top \in \mathbb{R}^{p_{(1)}\times p_{(2)}}$ and $Z_{22} \in \mathbb{R}^{p_{(2)}\times p_{(2)}}$.
        Then, $\tilde{Z} = [(\Xzt)^\top \Xzt]^{-1}$ is obtained by
\begin{equation} \label{zinverse}
 \tilde{Z} = \begin{bmatrix} U+UVWV^\top U & -UVW \\ -WV^\top U & W \end{bmatrix},
\end{equation}
        where $U = Z_{11}-Z_{12}Z_{22}^{-1}Z_{21}$, $V = X_{(1)}^\top \tilde{X}_{(2)}$ and $W = (\tilde{X}_{(2)}^\top \tilde{X}_{(2)}- V^\top UV)^{-1}$.
        Moreover, the computation of \eqref{zinverse} requires $\mathcal{O}(np_{(1)})$ time.
        %complexity to calculate $\tilde{Z}$ from $X_1,X_2,\tilde{X}_2$ and $Z$ is $\mathcal{O}((n+p_{(1)})p_{(1)})$.
\end{lemma}
\begin{proof}[proof]
From block matrix inversion of $\tilde{Z} = [(\Xzt)^\top \Xzt]^{-1} = ([X_{(1)},\tilde{X}_{(2)}]^\top [X_{(1)},\tilde{X}_{(2)}])^{-1}$, we yield
\begin{equation} \label{blockinverse}
 \tilde{Z} = \begin{bmatrix} A & V \\ V^\top & D \end{bmatrix}^{-1}=\begin{bmatrix} A^{-1}+A^{-1}VS^{-1}V^\top A^{-1} & -A^{-1}VS^{-1} \\ -S^{-1}V^\top A^{-1} & S^{-1} \end{bmatrix},
\end{equation}
where $A = X_{(1)}^\top X_{(1)}$, $V=X_{(1)}^\top \tilde{X}_{(2)}$, $D = \tilde{X}_{(2)}^\top \tilde{X}_{(2)}$ and $S=D-V^\top A^{-1}V$.
Since we also obtain $A = X_{(1)}^\top X_{(1)} = (Z_{11}-Z_{12}Z_{22}^{-1}Z_{21})^{-1}$ from block matrix inversion of
$Z = [(\Xz)^\top \Xz]^{-1} = ([X_{(1)},X_{(2)}]^\top [X_{(1)},X_{(2)}])^{-1}$,
we yield \eqref{zinverse} by substituting $U = A^{-1}$ and $W = S^{-1}$ into \eqref{blockinverse}.

Using $p_{(2)},\tilde{p}_{(2)}\le 2$, $U$, $V$ and $W$ can be obtained in
$\mathcal{O}(p_{(1)}^2)$ time, $\mathcal{O}(np_{(1)})$ time and $\mathcal{O}(p_{(1)}^2)$ time, respectively.
Therefore, because $n > p_{(1)}$, the computation of $\tilde{Z}$ requires $\mathcal{O}(np_{(1)})$ time.
\end{proof} 
From Lemma \ref{lemma3}, $[(\Xz)^\top \Xz]^{-1}$ can be updated in $\mathcal{O}(n\abs{\group})$ time
where $\abs{\group}$ is the number of fused groups which is equal to the number of columns in $\Xg$.
Moreover, because $n\ge p$, the complexity of updating the other variables 
$\group$, $\betag$, $\Xg$, $o(\cdot)$, $\frac{\D \betag}{\D \eta}$ and $\frac{\D \beta}{\D \eta}$
is no more than that of updating $[(\Xz)^\top \Xz]^{-1}$.
Thus, updating those variables when a fusing/splitting event occurs requires $\mathcal{O}(n\abs{\group})$ time.

After updating those variables, 
the computational cost to calculate the next timings of events $\Delta_{g}^\text{fuse}, \Delta_{g, k}^\text{split}$ and $\Delta_{k}^\text{switch}$
is evaluated as follows: %it requires $\mathcal{O}(np)$ time to compute of event times $\Delta_{g}^\text{fuse}, \Delta_{g, k}^\text{split}$ and $\Delta_{k}^\text{switch}$ after a fusing/splitting event because
\begin{itemize}
 \item All the timings of fusing events $\Delta_{g}^\text{fuse}$ can be obtained in $\mathcal{O}(\abs{\group})$ time.
 \item To obtain the timings of splitting/switching events,
 we need to calculate $\Xg \betag$ and $\Xg \frac{\D \betag}{\D \eta}$ which requires $\mathcal{O}(n \abs{\group})$ time.
 Then, given $\Xg \betag$ and $\Xg \frac{\D \betag}{\D \eta}$, each timing of splitting/switching events can be obtained in $\mathcal{O}(n)$ time. % Given $\Xg \betag$ and $\Xg \frac{\D \betag}{\D \eta}$ updated in $\mathcal{O}(n \abs{\group})$ time, each timing of splitting/switching events can be obtained in $\mathcal{O}(n)$ time.
 Thus, the computation of all the timings of splitting events $\Delta_{g, k}^\text{split}$ and switching events $\Delta_{k}^\text{switch}$ requires $\mathcal{O}(np)$ time.
\end{itemize}
Above all, it requires $\mathcal{O}(np)$ time to update all the variables and the next timings of events for a fusing/splitting event.

\subsection{Complexity per iteration for switching events}
Next, we discuss the complexity of the iteration where a switching event occurs. %computational cost after a fusing/splitting event is selected as the occurrence event in a iteration.

When the switching event which swaps the indices assigned to $o(k)$ and $o(k+1)$ in a group $G_g$ occurs,
the other variables $\group$, $\betag$, $\Xg$, $[(\Xz)^\top \Xz]^{-1}$, $\frac{\D \betag}{\D \eta}$ and $\frac{\D \beta}{\D \eta}$ than $o(k)$ and $o(k+1)$ 
are preserved as before the event.
As for the next timings of events, 
we need to calculate the following ones by their definition 
using $x_{o(k)}$, $x_{o(k+1)}$, $x_{\underline{o}(g,k)}$ and $x_{\overline{o}(g,k)}$ updated in the switching event. %after updating %which are involved in the switching event. 
\begin{itemize}
 \item $\Delta_{g,k}^{\text{split}}$ and, for $g=0$ in clustered Lasso, $\Delta_{g,-k}^{\text{split}}$. %because $x_{\underline{o}(g,k)}$ and $x_{\overline{o}(g,k)}$ change by the switch between $x_{o(k)}$ and $x_{o(k+1)}$.
 \item $\Delta_{k-1}^\text{switch}$, $\Delta_{k}^\text{switch}$ and $\Delta_{k+1}^\text{switch}$. %are changed  because of the switch between $x_{o(k)}$ and $x_{o(k+1)}$.
\end{itemize}
Each of them can be obtained in $\mathcal{O}(n)$ time.
The remainder of the next timings of events can be updated by only subtracting the step size $\Delta_{k}^\text{switch}$ of the current switching event from them,
which only requires $\mathcal{O}(p)$ time.

Thus, it requires $\mathcal{O}(n)$ time to update the variables and the next timings of events for a switching event of a pair of indices.
Similarly, the computation for the switching event which flips the sign of $s_1$ in OSCAR also requires $\mathcal{O}(n)$.

%\bibliographystyle{plain}
%\bibliography{nips_bunken_arxiv}

%\bibliographystyle{}% BibTeX を使う場合
%\bibliography{.bib ファイル名}% BibTeX を使う場合

\begin{thebibliography}{21}

\bibitem{arnold2016efficient}
Taylor~B Arnold and Ryan~J Tibshirani.
\newblock Efficient implementations of the generalized lasso dual path
  algorithm.
\newblock {\em Journal of Computational and Graphical Statistics}, 25(1):1--27,
  2016.

\bibitem{bertsekas1999nonlinear}
Dimitri~P Bertsekas.
\newblock {\em Nonlinear programming}.
\newblock Athena scientific Belmont, 1999.

\bibitem{bogdan2015}
Malgorzata Bogdan, Ewout van~den Berg, Chiara Sabatti, Weijie Su, and
  Emmanuel~J. Candes.
\newblock Slope-adaptive variable selection via convex optimization.
\newblock {\em The Annals of Applied Statistics}, 9(3):1103--1140, 09 2015.

\bibitem{bondell2008simultaneous}
Howard~D Bondell and Brian~J Reich.
\newblock Simultaneous regression shrinkage, variable selection, and supervised
  clustering of predictors with oscar.
\newblock {\em Biometrics}, 64(1):115--123, 2008.

\bibitem{LIBSVM2011}
Chih-Chung Chang and Chih-Jen Lin.
\newblock {LIBSVM} : a library for support vector machines.
\newblock {\em ACM Transactions on Intelligent Systems and Technology},
  2(3):1--27, 2011. Software available at
  \url{http://www.csie.ntu.edu.tw/~cjlin/libsvm}.

\bibitem{chiquet2017fast}
Julien Chiquet, Pierre Gutierrez, and Guillem Rigaill.
\newblock Fast tree inference with weighted fusion penalties.
\newblock {\em Journal of Computational and Graphical Statistics},
  26(1):205--216, 2017.

\bibitem{Dua:2019}
Dheeru Dua and Casey Graff.
\newblock {UCI} machine learning repository, 2019.
  \url{http://archive.ics.uci.edu/ml}.

\bibitem{gaines2018algorithms}
Brian~R Gaines, Juhyun Kim, and Hua Zhou.
\newblock Algorithms for fitting the constrained lasso.
\newblock {\em Journal of Computational and Graphical Statistics},
  27(4):861--871, 2018.

\bibitem{gu2017groups}
Bin Gu, Guodong Liu, and Heng Huang.
\newblock Groups-keeping solution path algorithm for sparse regression with
  automatic feature grouping.
\newblock In {\em Proceedings of the 23rd ACM SIGKDD International Conference
  on Knowledge Discovery and Data Mining}, pages 185--193. ACM, 2017.

\bibitem{hocking2011clusterpath}
Toby~Dylan Hocking, Armand Joulin, Francis Bach, and Jean-Philippe Vert.
\newblock Clusterpath: an algorithm for clustering using convex fusion
  penalties.
\newblock In {\em Proceedings of the 28th International Conference on Machine
  Learning}, pages 745--752, 2011.

\bibitem{hoefling2010path}
Holger Hoefling.
\newblock A path algorithm for the fused lasso signal approximator.
\newblock {\em Journal of Computational and Graphical Statistics},
  19(4):984--1006, 2010.

\bibitem{hu2015dual}
Qinqin Hu, Peng Zeng, and Lu~Lin.
\newblock The dual and degrees of freedom of linearly constrained generalized
  lasso.
\newblock {\em Computational Statistics \& Data Analysis}, 86:13--26, 2015.

\bibitem{lin2019efficient}
Meixia Lin, Yong-Jin Liu, Defeng Sun, and Kim-Chuan Toh.
\newblock Efficient sparse semismooth newton methods for the clustered lasso
  problem.
\newblock {\em SIAM Journal on Optimization}, 29(3):2026--2052, 2019.

\bibitem{luo2019efficient}
Ziyan Luo, Defeng Sun, Kim-Chuan Toh, and Naihua Xiu.
\newblock Solving the oscar and slope models using a semismooth newton-based
  augmented lagrangian method.
\newblock {\em Journal of Machine Learning Research}, 20(106):1--25, 2019.

\bibitem{she2010}
Yiyuan She.
\newblock Sparse regression with exact clustering.
\newblock {\em Electronic Journal of Statistics}, 4:1055--1096, 2010.

\bibitem{tarjan1983network}
Robert~E Tarjan.
\newblock {\em Data structures and network algorithms}.
\newblock Philadelphia, 1983.

\bibitem{tibshirani1996regression}
Robert Tibshirani.
\newblock Regression shrinkage and selection via the lasso.
\newblock {\em Journal of the Royal Statistical Society. Series B}, pages
  267--288, 1996.

\bibitem{tibshirani2005sparsity}
Robert Tibshirani, Michael Saunders, Saharon Rosset, Ji~Zhu, and Keith Knight.
\newblock Sparsity and smoothness via the fused lasso.
\newblock {\em Journal of the Royal Statistical Society: Series B},
  67(1):91--108, 2005.

\bibitem{tibshirani2011}
Ryan~J. Tibshirani and Jonathan Taylor.
\newblock The solution path of the generalized lasso.
\newblock {\em The Annals of Statistics}, 39(3):1335--1371, 06 2011.

\bibitem{zhong2012efficient}
Leon~Wenliang Zhong and James~T Kwok.
\newblock Efficient sparse modeling with automatic feature grouping.
\newblock {\em IEEE Transactions on Neural Networks and Learning Systems},
  23(9):1436--1447, 2012.

\bibitem{zhou2013algorithms}
Hua Zhou and Kenneth Lange.
\newblock A path algorithm for constrained estimation.
\newblock {\em Journal of Computational and Graphical Statistics},
  22(2):261--283, 2013.

\end{thebibliography}
\end{document}